\documentclass[english]{article}
\usepackage[T1]{fontenc}
\usepackage[latin9]{inputenc}
\usepackage{geometry}
\usepackage{babel}
\usepackage{verbatim}
\usepackage{mathtools}
\usepackage{amsmath}
\usepackage{amssymb}
\usepackage[unicode=true,
 bookmarks=false,
 breaklinks=false,pdfborder={0 0 1},colorlinks=false]
 {hyperref}
\hypersetup{
 colorlinks,linkcolor=blue,anchorcolor=blue,citecolor=blue}

\makeatletter
\usepackage{babel}
\usepackage{babel}
\usepackage{babel}

\usepackage{mathtools}

\usepackage{cite}\usepackage{amsthm}\usepackage{dsfont}\usepackage{array}\usepackage{mathrsfs}\usepackage{comment}\onecolumn
\usepackage{natbib}
\usepackage{color}\usepackage{babel}

\newcommand{\mymid}{\,|\,} 

\allowdisplaybreaks

\usepackage{enumitem}
\setlist[itemize]{leftmargin=1.5em}
\setlist[enumerate]{leftmargin=1.5em}

\usepackage{babel}
\usepackage{algorithm}
\usepackage{algorithmic}
\usepackage{arydshln}

\DeclareMathOperator{\ind}{\mathds{1}}  

\numberwithin{equation}{section}

\definecolor{yxc}{RGB}{255,0,0}
\definecolor{yjc}{RGB}{125,0,0}
\definecolor{cm}{RGB}{0,0,200}
\definecolor{yly}{RGB}{0,150,0}

\makeatother

\begin{document}
\theoremstyle{plain} \newtheorem{lemma}{\textbf{Lemma}} \newtheorem{prop}{\textbf{Proposition}}\newtheorem{theorem}{\textbf{Theorem}}\setcounter{theorem}{0}
\newtheorem{corollary}{\textbf{Corollary}} \newtheorem{assumption}{\textbf{Assumption}}
\newtheorem{example}{\textbf{Example}} \newtheorem{definition}{\textbf{Definition}}
\newtheorem{fact}{\textbf{Fact}} \newtheorem{condition}{\textbf{Condition}}\theoremstyle{definition}

\theoremstyle{remark}\newtheorem{remark}{\textbf{Remark}}\newtheorem{claim}{\textbf{Claim}}\newtheorem{conjecture}{\textbf{Conjecture}}
\title{Adapting to Unknown Low-Dimensional Structures in Score-Based Diffusion Models}
\author{Gen Li\footnote{The authors contributed equally.} \thanks{Department of Statistics, The Chinese University of Hong Kong, Hong
		Kong; Email: \texttt{genli@cuhk.edu.hk}.}\and Yuling Yan\footnotemark[1] \thanks{Department of Statistics, University of Wisconsin-Madison, WI 53706,
		USA; Email: \texttt{yuling.yan@wisc.edu}.}}

\maketitle
\begin{abstract}
This paper investigates score-based diffusion models when the underlying
target distribution is concentrated on or near low-dimensional manifolds
within the higher-dimensional space in which they formally reside,
a common characteristic of natural image distributions. Despite previous
efforts to understand the data generation process of diffusion models,
existing theoretical support remains highly suboptimal in the presence
of low-dimensional structure, which we strengthen in this paper. For
the popular Denoising Diffusion Probabilistic Model (DDPM), we find
that the dependency of the error incurred within each denoising step
on the ambient dimension $d$ is in general unavoidable. We further
identify a unique design of coefficients that yields a converges rate
at the order of $O(k^{2}/\sqrt{T})$ (up to log factors), where $k$
is the intrinsic dimension of the target distribution and $T$ is
the number of steps. This represents the first theoretical demonstration
that the DDPM sampler can adapt to unknown low-dimensional structures
in the target distribution, highlighting the critical importance of
coefficient design. All of this is achieved by a novel set of analysis
tools that characterize the algorithmic dynamics in a more deterministic
manner.

\end{abstract}


\setcounter{tocdepth}{2}

\tableofcontents{}

\section{Introduction}

Score-based diffusion models are a class of generative models that have gained prominence in the field of machine learning and artificial intelligence for their ability to generate high-quality new data instances from complex distributions, such as images, audio, and text \citep{sohl2015deep,ho2020denoising,song2020score,song2019generative,dhariwal2021diffusion}. These models operate by gradually transforming noise into samples from the target distribution through a denoising process guided by pre-trained neural networks that approximate the score functions. In practice, score-based diffusion models have demonstrated remarkable performance in generating realistic and diverse content across various domains \citep{ramesh2022hierarchical,rombach2022high,saharia2022photorealistic,croitoru2023diffusion}, achieving state-of-the-art performance in generative AI.

\subsection{Diffusion models}

The development of score-based diffusion models is deeply rooted in the theory of stochastic processes. At a high level, we consider a forward process:
\begin{equation}
	X_{0}\overset{\textsf{add noise}}{\longrightarrow}X_{1}\overset{\textsf{add noise}}{\longrightarrow}\cdots\overset{\textsf{add noise}}{\longrightarrow}X_{T},\label{eq:forward-process}
\end{equation}
which draws a sample from the target data distribution (i.e., $X_{0}\sim p_{\mathsf{data}}$), then progressively diffuses it to Gaussian noise over time. The key aspect of the diffusion model is to construct a reverse process:
\begin{equation}
	Y_{T}\overset{\textsf{denoise}}{\longrightarrow}Y_{T-1}\overset{\textsf{denoise}}{\longrightarrow}\cdots\overset{\textsf{denoise}}{\longrightarrow}Y_{0}\label{eq:reverse-process}
\end{equation}
satisfying $Y_{t}\overset{\mathsf{d}}{\approx}X_{t}$ for all $t$, which starts with pure Gaussian noise (i.e., $Y_{T}\sim\mathcal{N}(0,I_{d})$) and gradually converts it back to a new sample $Y_{0}$ sharing a similar distribution to $p_{\mathsf{data}}$.

The classical results on time-reversal of SDEs \citep{anderson1982reverse,haussmann1986time} provide the theoretical foundation for the above task. Consider a continuous time diffusion process:
\begin{equation}
\mathrm{d}X_{t}=-\frac{1}{2}\beta(t)X_{t}\mathrm{d}t+\sqrt{\beta(t)}\mathrm{d}W_{t}\quad(0\leq t\leq T),\qquad X_{0}\sim p_{\mathsf{data}}\label{eq:SDE}
\end{equation}
for some function $\beta:[0,T]\to\mathbb{R}^{+}$, where $(W_{t})_{0\leq t\leq T}$ is a standard Brownian motion. For a wide range of functions $\beta$, this process converges exponentially fast to a Gaussian distribution. Let $p_{X_{t}}(\cdot)$ be the density of $X_{t}$. One can construct a reverse-time SDE:
\begin{equation}
\mathrm{d}\widetilde{Y}_{t}=-\frac{1}{2}\beta(t)\big(\widetilde{Y}_{t}+2\nabla\log p_{X_{T-t}}(\widetilde{Y}_{t})\big)+\sqrt{\beta(t)}\mathrm{d}Z_{t}\quad(0\leq t\leq T),\qquad\widetilde{Y}_{0}\sim p_{X_{T}},\label{eq:reverse-SDE}
\end{equation}
where $(Z_{t})_{0\leq t\leq T}$ is another standard Brownian motion. Define $Y_{t}=\widetilde{Y}_{T-t}$. It is well-known that $X_{t}\overset{\mathsf{d}}{=}Y_{t}$ for all $0\leq t\leq T$. Here, $\nabla\log p_{X_{t}}$ is called the score function for the law of $X_{t}$, which is not explicitly known.

The above result motivates the following paradigm: we can construct the forward process (\ref{eq:forward-process}) by time-discretizing the diffusion process (\ref{eq:SDE}), and construct the reverse process (\ref{eq:reverse-process}) by discretizing the reverse-time SDE (\ref{eq:reverse-SDE}) and learning the score functions from the data. This approach leads to the popular DDPM sampler \citep{ho2020denoising,nichol2021improved}. Although the idea of the DDPM sampler is rooted in the theory of SDEs, the algorithm and analysis presented in this paper do not require any prior knowledge of SDEs.

This paper examines the accuracy of the DDPM sampler by establishing the proximity between the output distribution of the reverse process and the target data distribution. Since these two distributions are identical in the continuous time limit with perfect score estimation, the performance of the DDPM sampler is influenced by two sources of error: discretization error (due to a finite number of steps) and score estimation error (due to imperfect estimation of the scores). This paper views the score estimation step as a black box (often addressed by training a large neural network) and focuses on understanding how time discretization and imperfect score estimation affect the accuracy of the DDPM sampler.

\subsection{Inadequacy of existing results}

The past few years have witnessed a significant interest in studying the convergence guarantees for the DDPM sampler \citep{chen2022improved,chen2022sampling,benton2023linear,li2023towards}. To facilitate discussion, we consider an ideal setting with perfect score estimation. In this context, existing results can be interpreted as follows: to achieve $\varepsilon$-accuracy (i.e., the total variation distance between the target and the output distribution is smaller than $\varepsilon$), it suffices to take a number of steps exceeding the order of $\mathsf{poly}(d)/\varepsilon^{2}$ (ignoring logarithm factors), where $d$ is the problem dimension. Among these results, the state-of-the-art is given by \citet{benton2023linear}, which achieved linear dependency on the dimension $d$.

However, there seems to be a significant gap between the practical performance of the DDPM sampler and the existing theory. For example, for two widely used image datasets, CIFAR-10 (dimension $d=32\times32\times3$) and ImageNet (dimension $d\ge64\times64\times3$), it is known that 50 and 250 steps (also known as NFE, the number of function evaluations) are sufficient to generate good samples \citep{nichol2021improved,dhariwal2021diffusion}. This is in stark contrast with the existing theoretical guarantees discussed above, which suggest that the number of steps $T$ should exceed the order of the dimension $d$ to achieve good performance.

Empirical evidence suggests that the distributions of natural images are concentrated on or near low-dimensional manifolds within the higher-dimensional space in which they formally reside \citep{simoncelli2001natural,pope2021intrinsic}. In view of this, a reasonable conjecture is that the convergence rate of the DDPM sampler actually depends on the intrinsic dimension rather than the ambient dimension. However, the theoretical understanding of diffusion models when the support of the target data distribution has a low-dimensional structure remains vastly under-explored. As some recent attempts, \citet{de2022convergence} established the first convergence guarantee under the Wasserstein-1 metric. However, their error bound has linear dependence on the ambient dimension $d$ and exponential dependence on the diameter of the low-dimensional manifold. Another line of works \citep{chen2023score,tang2024adaptivity,oko2023diffusion} focused mainly on score estimation with properly chosen neural networks that exploit the low-dimensional structure, which is also different from our main focus.

\subsection{Our contributions}

In light of the large theory-practice gap and the insufficiency of prior results, this paper takes a step towards understanding the performance of the DDPM sampler when the target data distribution has low-dimensional structure. Our main contributions can be summarized as follows:
\begin{itemize}
\item We show that, with a particular coefficient design, the error of the
DDPM sampler, evaluated by the total variation distance between the
laws of $X_{1}$ and $Y_{1}$, is upper bounded by
\[
\frac{k^{2}}{\sqrt{T}}+\sqrt{\frac{1}{T}\sum_{t=1}^{T}\mathbb{E}\big[\left\Vert s_{t}\left(X_{t}\right)-s_{t}^{\star}\left(X_{t}\right)\right\Vert _{2}^{2}\big]},
\]
up to some logarithmic factors, where $k$ is the intrinsic dimension
of the target data distribution (which will be rigorously defined
later), and $s_{t}^{\star}$ (resp.~$s_{t}$) is the true (resp.~learned)
score function at each step. The first term represents the discretization
error (which vanishes as the number of steps $T$ goes to infinity),
while the second term should be interpreted as the score matching
error. This bound is nearly dimension-free --- the ambient dimension
$d$ only appears in logarithmic terms.
\item We also show that our choice of the coefficients is, in some sense,
the unique schedule that does not incur discretization error proportional
to the ambient dimension $d$ at each step. This is in sharp contrast
with the general setting without a low-dimensional structure, where
a fairly wide range of coefficient designs can lead to convergence
rates with polynomial dependence on $d$. Additionally, this confirms
the observation that the performance of the DDPM sampler can be improved through carefully designing coefficients \citep{bao2022analytic,nichol2021improved}. 
\end{itemize}
As far as we know, this paper provides the first theory demonstrating the capability of the DDPM sampler in adapting to unknown low-dimensional structures.

\section{Problem set-up} \label{sec:setup}

In this section, we introduce some preliminaries and key ingredients
for the diffusion model and the DDPM sampler.

\paragraph{Forward process.}

We consider the forward process (\ref{eq:forward-process}) of the
form
\begin{equation}
X_{t}=\sqrt{1-\beta_{t}}X_{t-1}+\sqrt{\beta_{t}}W_{t}\quad(t=1,\ldots,T),\qquad X_{0}\sim p_{\mathsf{data}},\label{eq:forward-update}
\end{equation}
where $W_{1},\ldots,W_{T}\overset{\text{i.i.d.}}{\sim}\mathcal{N}(0,I_{d})$,
and the learning rates $\beta_{t}\in(0,1)$ will be specified later.
For each $t\geq1$, $X_{t}$ has a probability density function (PDF)
supported on $\mathbb{R}^{d}$, and we will use $q_{t}$ to denote
the law or PDF of $X_{t}$. Let $\alpha_{t}\coloneqq1-\beta_{t}$
and $\overline{\alpha}_{t}\coloneqq\prod_{i=1}^{t}\alpha_{i}$. It
is straightforward to check that 
\begin{equation}
X_{t}=\sqrt{\overline{\alpha}_{t}}X_{0}+\sqrt{1-\overline{\alpha}_{t}}\,\overline{W}_{t}\qquad\text{where}\qquad\overline{W}_{t}\sim\mathcal{N}(0,I_{d}).\label{eq:forward-formula}
\end{equation}
We will choose the learning rates $\beta_{t}$ to ensure that $\overline{\alpha}_{T}$
becomes vanishingly small, such that $q_{T}\approx\mathcal{N}(0,I_{d})$. 

\paragraph{Score functions.}

The key ingredients for constructing the reverse process with the DDPM
sampler are the score functions $s_{t}^{\star}:\mathbb{R}^{d}\to\mathbb{R}^{d}$
associated with each $q_{t}$, defined as
\[
s_{t}^{\star}(x)\coloneqq\nabla\log q_{t}(x)\quad(t=1,\ldots,T).
\]
These score functions are not explicitly known. Here we assume access
to an estimate $s_{t}(\cdot)$ for each $s_{t}^{\star}(\cdot)$, and
we define the averaged $\ell_{2}$ score estimation error as
\[
\varepsilon_{\mathsf{score}}^{2}\coloneqq\frac{1}{T}\sum_{t=1}^{T}\mathbb{E}_{X\sim q_{t}}\left[\left\Vert s_{t}(X)-s_{t}^{\star}(X)\right\Vert _{2}^{2}\right].
\]
This quantity captures the effect of imperfect score estimation in
our theory.

\paragraph{The DDPM sampler.}

To construct the reverse process (\ref{eq:reverse-process}), we use
the DDPM sampler 
\begin{equation}
Y_{t-1}=\frac{1}{\sqrt{\alpha_{t}}}\big(Y_{t}+\eta_{t}s_{t}\left(Y_{t}\right)+\sigma_{t}Z_{t}\big)\quad(t=T,\ldots,1),\qquad Y_{T}\sim\mathcal{N}(0,I_{d})\label{eq:DDPM}
\end{equation}
where $Z_{1},\ldots,Z_{T}\overset{\text{i.i.d.}}{\sim}\mathcal{N}(0,I_{d})$.
Here $\eta_{t},\sigma_{t}>0$ are the hyperparameters that play an
important role in the performance of the DDPM sampler, especially
when the target data distribution has low-dimensional structure. As
we will see, our theory suggests the following choice
\begin{equation}
\eta_{t}^{\star}=1-\alpha_{t}\qquad\text{and}\qquad\sigma_{t}^{\star2}=\frac{\left(1-\alpha_{t}\right)\left(\alpha_{t}-\overline{\alpha}_{t}\right)}{1-\overline{\alpha}_{t}}.\label{eq:defn-step-size}
\end{equation}
For each $1\leq t\leq T$, we will use $p_{t}$ to denote the law
or PDF of $Y_{t}$. 

\paragraph{Target data distribution. }

Let $\mathcal{X}\subseteq\mathbb{R}^{d}$ be the support set of the
target data distribution $p_{\mathsf{data}}$, i.e., the smallest
closed set $C\subseteq\mathbb{R}^{d}$ such that $p_{\mathsf{data}}(C)=1$.
To allow for the greatest generality, we use the notion of $\varepsilon$-net
and covering number (see e.g., \citet{vershynin2018high}) to characterize
the intrinsic dimension of $\mathcal{X}$. For any $\varepsilon>0$,
a set $\mathcal{N}_{\varepsilon}\subseteq\mathcal{X}$ is said to
be an $\varepsilon$-net of $\mathcal{X}$ if for any $x\in\mathcal{X}$,
there exists some $x'$ in $\mathcal{N}_{\varepsilon}$ such that
$\Vert x-x'\Vert_{2}\leq\varepsilon$. The covering number $N_{\varepsilon}(\mathcal{X})$
is defined as the smallest possible cardinality of an $\varepsilon$-net
of $\mathcal{X}$. 
\begin{itemize}
\item (\textbf{Low-dimensionality}) Fix $\varepsilon=T^{-c_{\varepsilon}}$,
where $c_{\varepsilon}>0$ is some sufficiently large universal constant.
We define the intrinsic dimension of $\mathcal{X}$ to be some quantity
$k>0$ such that
\[
\log N_{\varepsilon}(\mathcal{X})\leq C_{\mathsf{cover}}k\log T
\]
for some constant $C_{\mathsf{cover}}>0$. 
\item (\textbf{Bounded support}) Suppose that there exists a universal
constant $c_{R}>0$ such that
\[
\sup_{x\in\mathcal{X}}\left\Vert x\right\Vert _{2}\leq R\qquad\text{where}\qquad R\coloneqq T^{c_{R}}.
\]
Namely we allow polynomial growth of the diameter of $\mathcal{X}$
in the number of steps $T$. 
\end{itemize}
Our setting allows $\mathcal{X}$ to be concentrated on or near low-dimensional
manifolds, which is less stringent than assuming an exact low-dimensional
structure. In fact, our definition of the intrinsic dimension $k$ is the metric entropy of $\mathcal{X}$ (see e.g., \citet{wainwright2019high}), which is widely used in statistics and learning theory to characterize the complexity of a set or a class. The low-dimensionality is also a concept of complexity, therefore it is natural to use covering number, or metric entropy to characterize the intrinsic dimension. As a sanity check, when $\mathcal{X}$ resides in an $r$-dimensional
subspace of $\mathbb{R}^{d}$, a standard volume argument (see e.g.,
\citet[Section 4.2.1]{vershynin2018high}) gives $\log N_{\varepsilon}(\mathcal{X})\asymp r\log(R/\varepsilon)\asymp r\log T$,
suggesting that the intrinsic dimension $k$ is of order $r$ in this
case. In addition, in applications like image generation, the data is naturally bounded, as pixel values are typically normalized within the range $[-1,1]$. For example, the $\ell_2$ norm of an image from the CIFAR dataset is typically below $60$.

\paragraph{Learning rate schedule.}

Following \citet{li2023towards}, we adopt the following learning rate
schedule
\begin{equation}\label{eq:learning-rate}
	\beta_{1}=\frac{1}{T^{c_{0}}},\qquad\beta_{t+1}=\frac{c_{1}\log T}{T}\min\left\{ \beta_{1}\left(1+\frac{c_{1}\log T}{T}\right)^{t},1\right\} \quad(t=1,\ldots,T-1)
\end{equation}
for some sufficiently large constants $c_{0},c_{1}>0$. This schedule
is not unique -- any other schedule of $\beta_{t}$ satisfying the
properties in Lemma~\ref{lemma:step-size} can lead to the same result
in this paper.

\section{Main results}

We are now positioned to present our main theoretical guarantees for
the DDPM sampler. 

\subsection{Convergence analysis \label{subsec:bound}}

We first present the convergence theory for the DDPM sampler. The
proof can be found in Section~\ref{sec:analysis}.

\begin{theorem} \label{thm:SDE} Suppose that we take the coefficients
for the DDPM sampler (\ref{eq:DDPM}) to be $\eta_{t}=\eta_{t}^{\star}$
and $\sigma_{t}=\sigma_{t}^{\star}$ (cf.~(\ref{eq:defn-step-size})),
then there exists some universal constant $C>0$ such that
\begin{equation}
\mathsf{TV}\left(q_{1},p_{1}\right)\leq C\frac{\left(k+\log d\right)^{2}\log^{3}T}{\sqrt{T}}+C\varepsilon_{\mathsf{score}}\log T.\label{eq:KL-bound}
\end{equation}
\end{theorem}

Several implications of Theorem~\ref{thm:SDE} follow immediately.
The two terms in (\ref{eq:KL-bound}) correspond to discretization
error and score matching error, respectively. Assuming perfect score
estimation (i.e.,~$\varepsilon_{\mathsf{score}}=0$) for the moment,
our error bound (\ref{eq:KL-bound}) suggests an iteration complexity
of order $k^{4}/\varepsilon^{2}$ (ignoring logarithmic factors) for
achieving $\varepsilon$-accuracy, for any nontrivial target accuracy
level $\varepsilon<1$. In the absence of low-dimensional structure
(i.e., $k\asymp d$), our result also recovers the iteration complexity
in \citet{chen2022improved,chen2022sampling,benton2023linear,li2023towards}
of order $\mathsf{poly}(d)/\varepsilon^{2}$.\footnote{Our result exhibits a quartic dimension dependency, which is worse than the linear dependency in \citet{benton2023linear}. This is mainly because we use a completely different analysis. It is not clear whether their analysis, which utilizes the SDE and stochastic localization toolbox, can tackle the problem with low-dimensional structure.}
This suggests that our choice of coefficients (\ref{eq:defn-step-size})
allows the DDPM sampler to adapt to any potential (unknown) low-dimensional
structure in the target data distribution, and remains a valid criterion
in the most general settings. The score matching error in (\ref{eq:KL-bound})
scales proportionally with $\varepsilon_{\mathsf{score}}$, suggesting
that the DDPM sampler is stable to imperfect score estimation. 

\subsection{Uniqueness of coefficient design \label{subsec:uniqueness}}

In this section, we examine the importance of the coefficient design
in the adaptivity of the DDPM sampler to intrinsic low-dimensional
structure. Our goal is to show that, unless the coefficients $\eta_{t},\sigma_{t}$
of the DDPM sampler (\ref{eq:DDPM}) are chosen according to (\ref{eq:defn-step-size}),
discretization errors proportional to the ambient dimension $d$ will
emerge in each denoising step. 

In this paper, as well as in most previous DDPM literature, the analysis
on the error $\mathsf{TV}(q_{1},p_{1})$ usually starts with the following
decomposition
\begin{align}
\mathsf{TV}^{2}(q_{1},p_{1}) & \overset{\text{(i)}}{\leq}\frac{1}{2}\mathsf{KL}\left(p_{X_{1}}\Vert p_{Y_{1}}\right)\overset{\text{(ii)}}{\leq}\frac{1}{2}\mathsf{KL}\left(p_{X_{1},\ldots,X_{T}}\Vert p_{Y_{1},\ldots,Y_{T}}\right)\nonumber \\
 & \overset{\text{(iii)}}{=}\frac{1}{2}\underbrace{\mathsf{KL}\left(p_{X_{T}}\Vert p_{Y_{T}}\right)}_{\text{initialization error}}+\frac{1}{2}\sum_{t=2}^{T}\underbrace{\mathbb{E}_{x_{t}\sim q_{t}}\left[\mathsf{KL}\left(p_{X_{t-1}|X_{t}}\left(\,\cdot\mymid x_{t}\right)\,\Vert\,p_{Y_{t-1}|Y_{t}}\left(\,\cdot\mymid x_{t}\right)\right)\right]}_{\text{error incurred in the }(T+1-t)\text{-th denoising step}}.\label{eq:error-decomposition}
\end{align}
Here step (i) follows from Pinsker's inequality, step (ii) utilizes
from the data-processing inequality, while step (iii) uses the chain
rule of KL divergence. We may interpret each term in the above decomposition
as the error incurred in each denoising step. In fact, this decomposition is also closely related to the variational bound on the negative log-likelihood of the reverse process, which is the optimization target for training DDPM \citep{ho2020denoising,bao2022analytic,nichol2021improved}.

We consider a target distribution $p_{\mathsf{data}}=\mathcal{N}(0,I_{k})$,
where $I_{k}\in\mathbb{R}^{d\times d}$ is a diagonal matrix with
$I_{i,i}=1$ for $1\leq i\leq k$ and $I_{i,i}=0$ for $k+1\leq i\leq d$.
This is a simple distribution over $\mathbb{R}^{d}$ that is supported
on a $k$-dimensional subspace.\footnote{Although this is not a bounded distribution, similar results can be established if we truncate $\mathcal{N}(0,I_{k})$ at the radius $R=T^{c_R}$. However  this is not essential and will make the result unnecessarily complicated, hence is omitted for clarity.}
Our second theoretical result provides a lower bound for the error
incurred in each denoising step for this target distribution. The
proof can be found in Appendix~\ref{sec:proof-thm-unique}. 

\begin{theorem}\label{thm:uniqueness} Consider the target distribution
$p_{\mathsf{data}}=\mathcal{N}(0,I_{k})$ and assume that $k\leq d/2$.
For the DDPM sampler (\ref{eq:DDPM}) with perfect score estimation
(i.e., $s_{t}(\cdot)=s_{t}^{\star}(\cdot)$ for all $t$) and arbitrary
coefficients $\eta_{t},\sigma_{t}>0$, we have
\[
\mathbb{E}_{x_{t}\sim q_{t}}\left[\mathsf{KL}\left(p_{X_{t-1}|X_{t}}\left(\,\cdot\mymid x_{t}\right)\,\Vert\,p_{Y_{t-1}|Y_{t}}\left(\,\cdot\mymid x_{t}\right)\right)\right]\geq\frac{d}{4}\left(\eta_{t}-\eta_{t}^{\star}\right)^{2}+\frac{d}{40}\left(\frac{\sigma_{t}^{\star2}}{\sigma_{t}^{2}}-1\right)^{2}
\]
for each $2\leq t\leq T$. See (\ref{eq:defn-step-size}) for the
definitions of $\eta_{t}^{\star}$ and $\sigma_{t}^{\star}$. \end{theorem}

Theorem~\ref{thm:uniqueness} shows that, unless we choose $\eta_{t}$
and $\sigma_{t}^{2}$ to be identical (or exceedingly close) to $\eta_{t}^{\star}$
and $\sigma_{t}^{\star2}$, the corresponding denoising step will
incur an undesired error that is linear in the ambient dimension $d$.
This highlights the critical importance of coefficient design for the DDPM sampler, especially when the target distribution exhibits a low-dimensional structure. 

Finally, we would like to make note that the above argument only
demonstrates the impact of coefficient design on an \textit{upper
bound} (\ref{eq:error-decomposition}) of the error $\mathsf{TV}(q_{1},p_{1})$,
rather than the error itself. It might be possible that a broader
range of coefficients can lead to dimension-independent error bound
like (\ref{eq:KL-bound}), while the upper bound (\ref{eq:error-decomposition})
remains dimension-dependent. This calls for new analysis tools (since we
cannot use the loose upper bound (\ref{eq:KL-bound}) in the analysis),
which we leave for future works.

\section{Analysis for the DDPM sampler (Proof of Theorem \ref{thm:SDE}) \label{sec:analysis}}

This section is devoted to establishing Theorem~\ref{thm:SDE}. The
idea is to bound the error incurred in each denoising step as characterized
in the decomposition (\ref{eq:error-decomposition}), namely for each
$2\leq t\leq T$, we need to bound
\[
\mathbb{E}_{x_{t}\sim q_{t}}\left[\mathsf{KL}\left(p_{X_{t-1}|X_{t}}\left(\,\cdot\mymid x_{t}\right)\,\Vert\,p_{Y_{t-1}|Y_{t}}\left(\,\cdot\mymid x_{t}\right)\right)\right].
\]
This requires connecting the two conditional distributions $p_{X_{t-1}|X_{t}}$
and $p_{Y_{t-1}|Y_{t}}$. It would be convenient to decouple the errors
from time discretization and imperfect score estimation by introducing
auxiliary random variables
\begin{equation}
Y_{t-1}^{\star}\coloneqq\frac{1}{\sqrt{\alpha_{t}}}\left(Y_{t}+\eta_{t}^{\star}s_{t}^{\star}\left(Y_{t}\right)+\sigma_{t}^{\star}Z_{t}\right)\qquad(2\leq t\leq T).\label{eq:Y-star-defn}
\end{equation}
On a high level, for each $2\leq t\leq T$, our proof consists of
the following steps:
\begin{enumerate}
\item Identify a typical set $\mathcal{A}_{t}\subseteq\mathbb{R}^{d}\times\mathbb{R}^{d}$
such that $(X_{t},X_{t-1})\in\mathcal{A}_{t}$ with high probability.
\item Establish point-wise proximity $p_{X_{t-1}|X_{t}}(x_{t-1}\mymid x_{t})\approx p_{Y_{t-1}^{\star}|Y_{t}}(x_{t-1}\mymid x_{t})$
for $(x_{t},x_{t-1})\in\mathcal{A}_{t}$. 
\item Characterize the deviation of $p_{Y_{t-1}^{\star}|Y_{t}}$ from $p_{Y_{t-1}|Y_{t}}$
caused by imperfect score estimation.
\end{enumerate}

\subsection{Step 1: identifying high-probability sets \label{subsec:prelim}}

For simplicity of presentation, we assume without loss of generality
that $k\geq\log d$ throughout the proof.\footnote{If $k<\log d$,
we may redefine $k\coloneqq\log d$, which does not change the desired
bound (\ref{eq:KL-bound}).} Let $\{x_{i}^{\star}\}_{1\leq i\leq N_{\varepsilon}}$
be an $\varepsilon$-net of $\mathcal{X}$, and let $\{\mathcal{B}_{i}\}_{1\leq i\leq N_{\varepsilon}}$
be a disjoint $\varepsilon$-cover for $\mathcal{X}$ such that $x_{i}^{\star}\in\mathcal{B}_{i}$.
Let
\begin{align*}
\mathcal{I} & \coloneqq\left\{ 1\leq i\leq N_{\varepsilon}:\mathbb{P}(X_{0}\in\mathcal{B}_{i})\geq\exp(-C_{1}k\log T)\right\} ,\\
\mathcal{G} & \coloneqq\big\{\omega\in\mathbb{R}^{d}:\Vert\omega\Vert_{2}\leq2\sqrt{d}+\sqrt{C_{1}k\log T},\quad\text{and}\\
 & \qquad\qquad\qquad\vert(x_{i}^{\star}-x_{j}^{\star})^{\top}\omega\vert\leq\sqrt{C_{1}k\log T}\Vert x_{i}^{\star}-x_{j}^{\star}\Vert_{2}\quad\text{for all}\quad1\leq i,j\leq N_{\varepsilon}\big\},
\end{align*}
where $C_{1}>0$ is some sufficiently large universal constants. Then
$\cup_{i\in\mathcal{I}}\mathcal{B}_{i}$ and $\mathcal{G}$ can be
interpreted as high probability sets for the variable $X_{0}$ and
a standard Gaussian random variable in $\mathbb{R}^{d}$. For each
$t=1,\ldots T$, we define a typical set for each $X_{t}$ as follows
\[
\mathcal{T}_{t}\coloneqq\left\{ \sqrt{\overline{\alpha}_{t}}x_{0}+\sqrt{1-\overline{\alpha}_{t}}\omega:x_{0}\in\cup_{i\in\mathcal{I}}\mathcal{B}_{i},\omega\in\mathcal{G}\right\} ,
\]
and a typical set for $(X_{t},X_{t-1})$ jointly as follows
\[
\mathcal{A}_{t}\coloneqq\Big\{\left(x_{t},x_{t-1}\right):x_{t}\in\mathcal{T}_{t},\frac{x_{t}-\sqrt{\alpha_{t}}x_{t-1}}{\sqrt{1-\alpha_{t}}}\in\mathcal{G}\Big\}.
\]
The following lemma shows that $\mathcal{A}_{t}$ is indeed a high-probability
set for $(X_{t},X_{t-1})$. 

\begin{lemma}\label{lemma:At-SDE} Suppose that $C_{1}\gg C_{\mathsf{cover}}$.
Then for each $1\leq t\leq T$ we have
\[
\mathbb{P}\left(\left(X_{t},X_{t-1}\right)\notin\mathcal{A}_{t}\right)\leq\exp\Big(-\frac{C_{1}}{4}k\log T\Big).
\]
\end{lemma}\begin{proof}See Appendix~\ref{subsec:proof-lemma-At-SDE}.\end{proof}

\subsection{Step 2: connecting conditional densities $p_{X_{t-1}|X_{t}}$ and
$p_{Y_{t-1}^{\star}|Y_{t}}$}

Given the definition of $Y_{t-1}^{\star}$ in (\ref{eq:Y-star-defn}),
we can write down the conditional density $p_{Y_{t-1}^{\star}|Y_{t}}$
as follows
\begin{align}
p_{Y_{t-1}^{\star}|Y_{t}}\left(x_{t-1}\mymid x_{t}\right) & =\left(\frac{\alpha_{t}}{2\pi\sigma_{t}^{\star2}}\right)^{d/2}\exp\left(-\frac{\Vert\sqrt{\alpha_{t}}x_{t-1}-x_{t}-\eta_{t}^{\star}s_{t}^{\star}\left(x_{t}\right)\Vert_{2}^{2}}{2\sigma_{t}^{\star2}}\right).\label{eq:proof-sde-8}
\end{align}
Next, we will investigate the conditional density $p_{X_{t-1}|X_{t}}$
for the forward process. For each $x_{0}\in\mathcal{X}$, we define
the shorthand notation 
\begin{equation}
\widehat{x}_{0}\coloneqq\mathbb{E}\left[X_{0}\mymid X_{t}=x_{t}\right]=\int_{x_{0}}x_{0}p_{X_{0}|X_{t}}\left(x_{0}\mymid x_{t}\right)\mathrm{d}x_{0},\label{eq:x0-hat-defn}
\end{equation}
and define a function $\Delta_{x_{t},x_{t-1}}:\mathcal{X}\to\mathbb{R}$
as follows
\begin{align}
\Delta_{x_{t},x_{t-1}}\left(x_{0}\right) & \coloneqq-\frac{\sqrt{\overline{\alpha}_{t}}}{\alpha_{t}-\overline{\alpha}_{t}}\left(\sqrt{\alpha_{t}}x_{t-1}-x_{t}\right)^{\top}\left(\widehat{x}_{0}-x_{0}\right)
-\frac{\left(1-\alpha_{t}\right)\overline{\alpha}_{t}}{2\left(\alpha_{t}-\overline{\alpha}_{t}\right)\left(1-\overline{\alpha}_{t}\right)}\Vert\widehat{x}_{0}-x_{0}\Vert_{2}^{2}\nonumber \\
 & \qquad
 -\frac{\left(1-\alpha_{t}\right)\sqrt{\overline{\alpha}_{t}}}{\left(\alpha_{t}-\overline{\alpha}_{t}\right)\left(1-\overline{\alpha}_{t}\right)}\left(x_{t}-\sqrt{\overline{\alpha}_{t}}\widehat{x}_{0}\right)^{\top}\left(\widehat{x}_{0}-x_{0}\right).\label{eq:Delta-defn}
\end{align}
The next lemma provides a characterization for $p_{X_{t-1}|X_{t}}$
that shows an explicit connection with $p_{Y_{t-1}^{\star}|Y_{t}}$.

\begin{lemma}\label{lemma:cond-density} For any pair $(x_{t},x_{t-1})\in\mathbb{R}^{d}\times\mathbb{R}^{d}$,
we have
\begin{align*}
p_{X_{t-1}|X_{t}}\left(x_{t-1}\mymid x_{t}\right) & =\left(\frac{\alpha_{t}}{2\pi\sigma_{t}^{\star2}}\right)^{d/2}\exp\left(-\frac{\Vert\sqrt{\alpha_{t}}x_{t-1}-x_{t}-\eta_{t}^{\star}s_{t}^{\star}\left(x_{t}\right)\Vert_{2}^{2}}{2\sigma_{t}^{\star2}}\right)\\
 & \qquad\cdot\int_{\mathcal{X}}\exp\left(\Delta_{x_{t},x_{t-1}}\left(x_{0}\right)\right)p_{X_{0}|X_{t}}\left(x_{0}\mymid x_{t}\right)\mathrm{d}x_{0}.
\end{align*}

\end{lemma}

\begin{proof}See Appendix~\ref{subsec:proof-lemma-cond-density}.\end{proof}

Taking Lemma~\ref{lemma:cond-density} and (\ref{eq:proof-sde-8})
collectively yields
\begin{align*}
\frac{p_{X_{t-1}|X_{t}}\left(x_{t-1}\mymid x_{t}\right)}{p_{Y_{t-1}^{\star}|Y_{t}}\left(x_{t-1}\mymid x_{t}\right)} & =\int_{\mathcal{X}}\exp\left(\Delta_{x_{t},x_{t-1}}\left(x_{0}\right)\right)p_{X_{0}|X_{t}}\left(x_{0}\mymid x_{t}\right)\mathrm{d}x_{0},
\end{align*}
which allows us to control the density ratio by the magnitude of $\Delta_{x_{t},x_{t-1}}$.
By a careful analysis of the above integral for all $(x_{t},x_{t-1})\in\mathcal{A}_{t}$,
we show in the next lemma that the density ratio is uniformly close
to $1$ within the typical set $\mathcal{A}_{t}$. 

\begin{lemma}\label{lemma:Delta-SDE}Suppose that $T\gg k^{2}\log^{3}T$.
Then there exists some universal constant $C_{5}>0$ such that, for
any $2\leq t\leq T$ and any $(x_{t},x_{t-1})\in\mathcal{A}_{t}$,
we have
\[
\left|\frac{p_{X_{t-1}|X_{t}}\left(x_{t-1}\mymid x_{t}\right)}{p_{Y_{t-1}^{\star}|Y_{t}}\left(x_{t-1}\mymid x_{t}\right)}-1\right|\leq C_{5}\frac{k^{2}\log^{3}T}{T}\leq\frac{1}{2}.
\]
 \end{lemma} \begin{proof}See Appendix~\ref{subsec:proof-lemma-Delta-SDE}.\end{proof}

For $(x_{t},x_{t-1})$ outside the typical set $\mathcal{A}_{t}$,
the following lemma gives a coarse uniform bound for the density ratio,
which is already sufficient for our later analysis.

\begin{lemma}\label{lemma:sde-coarse}Suppose that $T\gg1$. Then
for any $2\leq t\leq T$ and any pair $(x_{t},x_{t-1})\in\mathbb{R}^{d}\times\mathbb{R}^{d}$,
we have
\[
\left|\log\frac{p_{X_{t-1}|X_{t}}\left(x_{t-1}\mymid x_{t}\right)}{p_{Y_{t-1}^{\star}|Y_{t}}\left(x_{t-1}\mymid x_{t}\right)}\right|\leq T^{c_{0}+2c_{R}}\left(\Vert\sqrt{\alpha_{t}}x_{t-1}-x_{t}\Vert_{2}+\Vert x_{t}\Vert_{2}+1\right).
\]
\end{lemma}

\begin{proof} See Appendix~\ref{subsec:proof-lemma-sde-coarse}.\end{proof}

Armed with Lemmas \ref{lemma:Delta-SDE} and \ref{lemma:sde-coarse},
we are ready to bound the expected KL divergence between the two conditional
distributions $p_{X_{t-1}|X_{t}}$ and $p_{Y_{t-1}^{\star}|Y_{t}}$.

\subsection{Step 3: bounding the KL divergence between $p_{X_{t-1}|X_{t}}$ and
$p_{Y_{t-1}^{\star}|Y_{t}}$}

We first decompose the expected KL divergence between $p_{X_{t-1}|X_{t}}$
and $p_{Y_{t-1}^{\star}|Y_{t}}$ into
\begin{align*}
 & \mathbb{E}_{x_{t}\sim q_{t}}\left[\mathsf{KL}\left(p_{X_{t-1}|X_{t}}\left(\,\cdot\mymid x_{t}\right)\,\Vert\,p_{Y_{t-1}^{\star}|Y_{t}}\left(\,\cdot\mymid x_{t}\right)\right)\right]\\
 & \quad=\left(\int_{\mathcal{A}_{t}}+\int_{\mathcal{A}_{t}^{\mathrm{c}}}\right)p_{X_{t-1}|X_{t}}\left(x_{t-1}\mymid x_{t}\right)\log\left(\frac{p_{X_{t-1}|X_{t}}\left(x_{t-1}\mymid x_{t}\right)}{p_{Y_{t-1}^{\star}|Y_{t}}\left(x_{t-1}\mymid x_{t}\right)}\right)p_{X_{t}}\left(x_{t}\right)\mathrm{d}x_{t-1}\mathrm{d}x_{t}\\
 & \quad \eqqcolon\Delta_{t,1}+\Delta_{t,2},
\end{align*}
where $\Delta_{t,1}$ and $\Delta_{t,2}$ are the integrals over $\mathcal{A}_{t}$
and $\mathcal{A}_{t}^{\mathrm{c}}$. It boils down to bounding these
two terms. 

By a direct application of Lemma~\ref{lemma:Delta-SDE} together
with the first-order Taylor expansion of $\log(x)$ around $x=1$,
one can easily show that $|\Delta_{t,1}|\lesssim k^{2}\log^{3}(T)/T$.
However this naive bound will lead to a vacuous final bound on $\mathsf{TV}(q_{1},p_{1})$,
which depends on the sum of $\Delta_{t,1}$ over all $2\leq t\leq T$
according to (\ref{eq:error-decomposition}). By a more careful analysis,
we achieve a better bound for $\Delta_{t,1}$, as shown in the following
lemma.

\begin{lemma}\label{lemma:SDE-Delta-t-1}Suppose that $T\gg k^{2}\log^{3}T$.
Then for each $2\leq t\leq T$, we have
\[
\left|\Delta_{t,1}\right|\leq2C_{5}^{2}\frac{k^{4}\log^{6}T}{T^{2}}.
\]
\end{lemma}

\begin{proof}See Appendix~\ref{subsec:proof-lemma-sde-delta-t-1}.\end{proof}

For $\Delta_{t,2}$, we can employ the course bound in Lemma~\ref{lemma:sde-coarse}
to show that it is exponentially small.

\begin{lemma}\label{lemma:SDE-Delta-t-2}Suppose that $T\gg1$. Then
for each $2\leq t\leq T$, we have
\[
\left|\Delta_{t,2}\right|\leq\exp\left(-\frac{C_{1}}{16}k\log T\right).
\]
\end{lemma}

\begin{proof}See Appendix~\ref{subsec:proof-lemma-sde-delta-t-2}.\end{proof}

By putting together Lemma~\ref{lemma:SDE-Delta-t-1} and Lemma~\ref{lemma:SDE-Delta-t-2},
we achieve
\begin{align}
\mathbb{E}_{x_{t}\sim q_{t}}\left[\mathsf{KL}\left(p_{X_{t-1}|X_{t}}\left(\,\cdot\mymid x_{t}\right)\,\Vert\,p_{Y_{t-1}^{\star}|Y_{t}}\left(\,\cdot\mymid x_{t}\right)\right)\right] & =\Delta_{t,1}+\Delta_{t,2}\leq3C_{5}^{2}\frac{k^{4}\log^{6}T}{T^{2}}\label{eq:proof-KL-main-1}
\end{align}
provided that $T$ is sufficiently large. 

\subsection{Step 4: bounding the KL divergence between $p_{X_{t-1}|X_{t}}$ and
$p_{Y_{t-1}|Y_{t}}$}

Since our goal is to bound the expected KL divergence between $p_{X_{t-1}|X_{t}}$
and $p_{Y_{t-1}|Y_{t}}$, we also need to upper bound the following
difference
\begin{align}
 & \mathbb{E}_{x_{t}\sim q_{t}}\Big[\mathsf{KL}\big(p_{X_{t-1}|X_{t}}\left(\,\cdot\mymid x_{t}\right)\,\Vert\,p_{Y_{t-1}|Y_{t}}\left(\,\cdot\mymid x_{t}\right)\big)\Big]-\mathbb{E}_{x_{t}\sim q_{t}}\Big[\mathsf{KL}\big(p_{X_{t-1}|X_{t}}\left(\,\cdot\mymid x_{t}\right)\,\Vert\,p_{Y_{t-1}^{\star}|Y_{t}}\left(\,\cdot\mymid x_{t}\right)\big)\Big]\nonumber\\
 & \quad=\int\bigg[\int p_{X_{t-1}|X_{t}}\left(x_{t-1}\mymid x_{t}\right)\log\frac{p_{Y_{t-1}^{\star}|Y_{t}}\left(x_{t-1}\mymid x_{t}\right)}{p_{Y_{t-1}|Y_{t}}\left(x_{t-1}\mymid x_{t}\right)}\mathrm{d}x_{t-1}\bigg]q_{t}\left(x_{t}\right)\mathrm{d}x_{t} \label{eq:proof-KL-main-1.5} \\
 & \quad=\int p_{X_{t-1},X_{t}}\left(x_{t-1},x_{t}\right)\left(-\frac{\alpha_{t}\Vert x_{t-1}-\mu_{t}^{\star}\left(x_{t}\right)\Vert_{2}^{2}}{2\sigma_{t}^{\star2}}+\frac{\alpha_{t}\Vert x_{t-1}-\mu_{t}\left(x_{t}\right)\Vert_{2}^{2}}{2\sigma_{t}^{\star2}}\right)\mathrm{d}x_{t-1}\mathrm{d}x_{t}\nonumber\\
 & \quad=\frac{\eta_{t}^{\star2}}{2\sigma_{t}^{\star2}}\mathbb{E}_{x_{t}\sim q_{t}}\left[\Vert\varepsilon_{t}\left(x_{t}\right)\Vert_{2}^{2}\right]+\frac{\eta_{t}^{\star}\sqrt{\alpha_{t}}}{\sigma_{t}^{\star2}}\underbrace{\int p_{X_{t-1},X_{t}}\left(x_{t-1},x_{t}\right)\left(x_{t-1}-\mu_{t}^{\star}\left(x_{t}\right)\right)^{\top}\varepsilon_{t}\left(x_{t}\right)\mathrm{d}x_{t-1}\mathrm{d}x_{t}}_{\eqqcolon K_{t}},\nonumber
\end{align}
where we define
\begin{equation}
\varepsilon_{t}\left(x_{t}\right)\coloneqq s_{t}^{\star}\left(x_{t}\right)-s_{t}\left(x_{t}\right),\quad\mu_{t}^{\star}\left(x_{t}\right)\coloneqq\frac{x_{t}+\eta_{t}^{\star}s_{t}^{\star}\left(x_{t}\right)}{\sqrt{\alpha_{t}}},\quad\text{and}\quad\mu_{t}\left(x_{t}\right)\coloneqq\frac{x_{t}+\eta_{t}^{\star}s_{t}\left(x_{t}\right)}{\sqrt{\alpha_{t}}}.\label{eq:defn-epsilon-mu}
\end{equation}
It then boils down to bounding $K_{t}$, which is presented in the
following lemma. 

\begin{lemma}\label{lemma:sde-K}Suppose that $T\gg k^{2}\log^{3}T$.
Then we have
\[
\left|K_{t}\right|\leq4C_{5}\frac{k^{2}\log^{3}T}{T}\sqrt{\frac{c_{1}\log T}{T}}\mathbb{E}_{x_{t}\sim q_{t}}^{1/2}\left[\Vert\varepsilon_{t}\left(x_{t}\right)\Vert_{2}^{2}\right].
\]
\end{lemma}\begin{proof}See Appendix~\ref{subsec:proof-lemma-sde-delta-t-2}.\end{proof}

Hence we know that for $2\leq t\leq T$, 
\begin{align}
	& \mathbb{E}_{x_{t}\sim q_{t}}\left[\mathsf{KL}\left(p_{X_{t-1}|X_{t}}\left(\,\cdot\mymid x_{t}\right)\,\Vert\,p_{Y_{t-1}|Y_{t}}\left(\,\cdot\mymid x_{t}\right)\right)\right]-\mathbb{E}_{x_{t}\sim q_{t}}\big[\mathsf{KL}\big(p_{X_{t-1}|X_{t}}\left(\,\cdot\mymid x_{t}\right)\,\Vert\,p_{Y_{t-1}^{\star}|Y_{t}}\left(\,\cdot\mymid x_{t}\right)\big)\big]\nonumber \\
	& \quad\leq\frac{\left(1-\overline{\alpha}_{t}\right)\left(1-\alpha_{t}\right)}{2\left(\alpha_{t}-\overline{\alpha}_{t}\right)}\mathbb{E}_{x_{t}\sim q_{t}}\left[\Vert\varepsilon_{t}\left(x_{t}\right)\Vert_{2}^{2}\right]+4C_{5}\frac{1-\overline{\alpha}_{t}}{\alpha_{t}-\overline{\alpha}_{t}}\frac{k^{2}\log^{3}T}{T}\sqrt{\frac{c_{1}\log T}{T}}\mathbb{E}_{x_{t}\sim q_{t}}^{1/2}\left[\Vert\varepsilon_{t}\left(x_{t}\right)\Vert_{2}^{2}\right]\nonumber \\
	& \quad\leq\frac{4c_{1}\log T}{T}\mathbb{E}_{x_{t}\sim q_{t}}\left[\Vert\varepsilon_{t}\left(x_{t}\right)\Vert_{2}^{2}\right]+8C_{5}\frac{k^{2}\log^{3}T}{T}\sqrt{\frac{c_{1}\log T}{T}}\mathbb{E}_{x_{t}\sim q_{t}}^{1/2}\left[\Vert\varepsilon_{t}\left(x_{t}\right)\Vert_{2}^{2}\right].\label{eq:proof-KL-main-2}
\end{align}
Here the first relation follows from Lemma~\ref{lemma:sde-K} and
(\ref{eq:defn-step-size}); while the second relation follows from
Lemma~\ref{lemma:step-size} and holds provided that $T$ is sufficiently
large.

\subsection{Step 5: putting everything together}

By taking (\ref{eq:proof-KL-main-1}) and (\ref{eq:proof-KL-main-2})
collectively, we have
\begin{align}
 & \mathbb{E}_{x_{t}\sim q_{t}}\left[\mathsf{KL}\left(p_{X_{t-1}|X_{t}}\left(\,\cdot\mymid x_{t}\right)\,\Vert\,p_{Y_{t-1}|Y_{t}}\left(\,\cdot\mymid x_{t}\right)\right)\right]\nonumber \\
 & \quad\leq3C_{5}^{2}\frac{k^{4}\log^{6}T}{T^{2}}+\frac{4c_{1}\log T}{T}\mathbb{E}_{x_{t}\sim q_{t}}\left[\Vert\varepsilon_{t}\left(x_{t}\right)\Vert_{2}^{2}\right]+8C_{5}\frac{k^{2}\log^{3}T}{T}\sqrt{\frac{c_{1}\log T}{T}}\mathbb{E}_{x_{t}\sim q_{t}}^{1/2}\left[\Vert\varepsilon_{t}\left(x_{t}\right)\Vert_{2}^{2}\right]\nonumber \\
 & \quad\leq7C_{5}^{2}\frac{k^{4}\log^{6}T}{T^{2}}+\frac{8c_{1}\log T}{T}\mathbb{E}_{x_{t}\sim q_{t}}\left[\Vert\varepsilon_{t}\left(x_{t}\right)\Vert_{2}^{2}\right].\label{eq:proof-KL-main-3}
\end{align}
Here the last relation follows from an application of the AM-GM inequality
\[
8C_{5}\frac{k^{2}\log^{3}T}{T}\sqrt{\frac{c_{1}\log T}{T}}\mathbb{E}_{x_{t}\sim q_{t}}^{1/2}\left[\Vert\varepsilon_{t}\left(x_{t}\right)\Vert_{2}^{2}\right]\leq\frac{4c_{1}\log T}{T}\mathbb{E}_{x_{t}\sim q_{t}}\left[\Vert\varepsilon_{t}\left(x_{t}\right)\Vert_{2}^{2}\right]+4C_{5}^{2}\frac{k^{4}\log^{6}T}{T^{2}}.
\]
Finally we conclude that
\begin{align*}
\mathsf{TV}^{2}(q_{1},p_{1}) & \leq\mathsf{KL}\left(p_{X_{T}}\Vert p_{Y_{T}}\right)+\sum_{t=2}^{T}\mathbb{E}_{x_{t}\sim q_{t}}\left[\mathsf{KL}\left(p_{X_{t-1}|X_{t}}\left(\,\cdot\mymid x_{t}\right)\,\Vert\,p_{Y_{t-1}|Y_{t}}\left(\,\cdot\mymid x_{t}\right)\right)\right]\\
 & \leq8C_{5}^{2}\frac{k^{4}\log^{6}T}{T}+\frac{8c_{1}\log T}{T}\sum_{t=2}^{T}\mathbb{E}_{x_{t}\sim q_{t}}\left[\Vert\varepsilon_{t}\left(x_{t}\right)\Vert_{2}^{2}\right],
\end{align*}
as claimed. Here the first relation follows from (\ref{eq:error-decomposition}),
while the second relation follows from the fact that $\text{\ensuremath{\mathsf{KL}}(\ensuremath{p_{X_{T}}\Vert p_{Y_{T}}})\ensuremath{\leq T^{-100}}}$
provided that $T$ is sufficiently large (see Lemma~\ref{lemma:initialization-error}).
\section{Simulation study}

\begin{figure}[!h]
	\begin{center}
		\includegraphics[width=\textwidth]{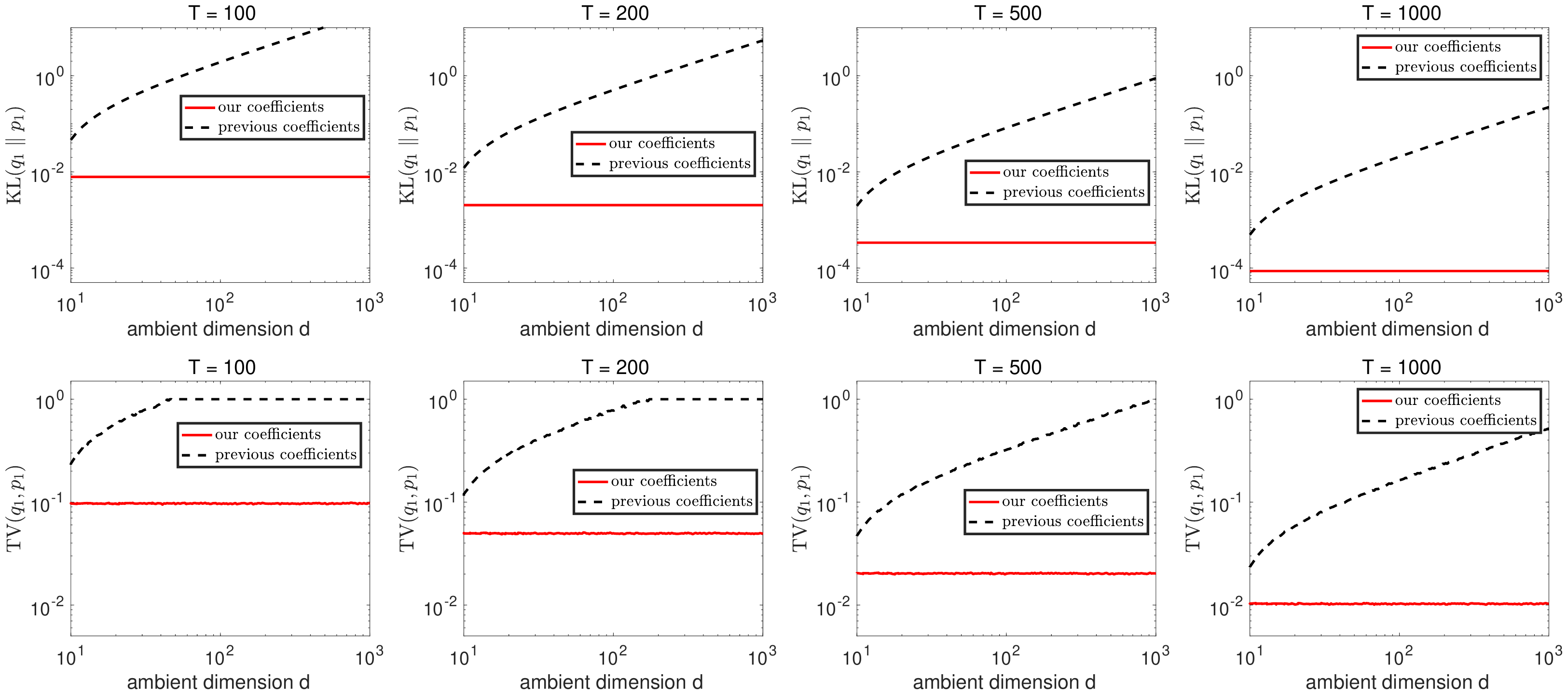} 
	\end{center}
	\caption{
		The KL divergence between $q_1$ and $p_1$ for $T \in \{100,200,500,1000\}$, when $p_\mathsf{data}=\mathcal{N}(0,I_k)$.
		We fix the low intrinsic dimension $k=8$, and let the ambient dimension $d$ grow from $10$ to $1000$.
	}
	\label{fig}
\end{figure}

We conducted a simple simulation to compare our coefficient design \eqref{eq:defn-step-size} with another design 
\begin{equation}\label{eq:other-design}
	\eta_t = \sigma_t^2 = 1-\alpha_t \qquad \text{for} \qquad 1\leq t \leq T,
\end{equation} 
which has been widely adopted in theoretical analysis of diffusion model (see e.g., \citet{li2023towards,li2024d}). We consider the degenerated Gaussian distribution $p_\mathsf{data}=\mathcal{N}(0,I_k)$ in Theorem~\ref{thm:uniqueness} as a tractable example, and run the DDPM sampler with exact score functions (so that the error only comes from discretization). We fix the intrinsic dimension $k=8$, and let the ambient dimension $d$ grow from $10$ to $10^3$. We implement the experiment for four different number of steps $T \in \{100,200,500,1000\}$. Instead of using the learning rate schedule \eqref{eq:learning-rate}, which is chosen mainly to facilitate analysis, we use the schedule in \citet{ho2020denoising} that is commonly used in practice. Figure~\ref{fig} displays the error, in terms of both the TV distance $\mathsf{TV}(q_1,p_1)$ and KL divergence $\mathsf{KL}(q_1 \Vert p_1)$, as the ambient dimension $d$ varies. As we can see, our design \eqref{eq:defn-step-size} leads to dimension-independent error while the other design \eqref{eq:other-design} incures an error that grows as $d$ increases. This provides empirical evidence that \eqref{eq:defn-step-size} represents a unique coefficient design for DDPM in achieving dimension-independent error.

\section{Discussion}
\label{sec:discuss}

The present paper investigates the DDPM sampler when the target distribution
is concentrated on or near low-dimensional manifolds. We identify
a particular coefficient design that enables the adaptivity of the
DDPM sampler to unknown low-dimensional structures and establish a
dimension-free convergence rate at the order of $k^{2}/\sqrt{T}$
(up to logarithmic factors). We conclude this paper by pointing out several
directions worthy of future investigation. To begin with, our theory
yields an iteration complexity that scales quartically in the intrinsic
dimension $k$, which is likely sub-optimal. Improving this dependency
calls for more refined analysis tools. Recent work \citep{li2024d} achieved a convergence rate of order $O(d/T)$, suggesting the potential for enhancing the dependence on $T$. Furthermore, as we have discussed
in the end of Section~\ref{subsec:uniqueness}, it is not clear whether
our coefficient design (\ref{eq:defn-step-size}) is unique in terms
of achieving dimension-independent error $\mathsf{TV}(q_{1},p_{1})$.
Finally, the analysis ideas and tools developed for the DDPM sampler
might be extended to study another popular DDIM sampler.

\section*{Acknowledgements}

Gen Li is supported in part by the Chinese University of Hong Kong Direct Grant for Research.
Yuling Yan was supported in part by a Norbert Wiener Postdoctoral Fellowship from MIT.

\appendix

\section{Proof of auxiliary lemmas for the DDPM sampler}

\subsection{Preliminaries \label{subsec:auxiliary-prelim}}

Fix any $x_{t}\in\mathcal{T}_{t}$, there exists an index $i(x_{t})\in\mathcal{I}$,
two points $x_{0}(x_{t})\in\mathcal{B}_{i(x_{t})}$ and $\omega\in\mathcal{G}$
such that 
\begin{equation}
x_{t}=\sqrt{\overline{\alpha}_{t}}x_{0}(x_{t})+\sqrt{1-\overline{\alpha}_{t}}\omega.\label{eq:proof-lemma-set-decom}
\end{equation}
For any $r>0$, define a set
\begin{equation}
\mathcal{I}\left(x_{t};r\right)\coloneqq\left\{ 1\leq i\leq N_{\varepsilon}:\overline{\alpha}_{t}\Vert x_{i}^{\star}-x_{i(x_{t})}^{\star}\Vert_{2}^{2}\leq r\cdot k(1-\overline{\alpha}_{t})\log T\right\} .\label{eq:defn-I-r}
\end{equation}
For some sufficiently large constant $C_{3}>0$, define
\begin{align*}
\mathcal{X}_{t}\left(x_{t}\right) & \coloneqq\bigcup_{i\in\mathcal{I}(x_{t};C_{3})}\mathcal{B}_{i}\qquad\text{and}\qquad\mathcal{Y}_{t}\left(x_{t}\right)\coloneqq\bigcup_{i\notin\mathcal{I}(x_{t};C_{3})}\mathcal{B}_{i}.
\end{align*}
Namely, $\mathcal{X}_{t}(x_{t})$ (resp.~$\mathcal{Y}_{t}(x_{t})$)
contains the indices of the $\varepsilon$-covering that are close
(resp.~far) from $\mathcal{B}_{i(x_{t})}$. We require that
\begin{equation}
\varepsilon\ll\sqrt{\frac{1-\overline{\alpha}_{t}}{\overline{\alpha}_{t}}}\min\left\{ 1,\sqrt{\frac{k\log T}{d}}\right\} ,\label{eq:eps-condition}
\end{equation}
which is guaranteed by our assumption that $\varepsilon=T^{-c_{\varepsilon}}$
for some sufficiently large constant $c_{\varepsilon}>0$. Under this
condition, for any $x,x'\in\mathcal{X}_{t}(x_{t})$ we have
\[
\Vert x-x'\Vert_{2}\leq\Vert x-x_{i(x_{t})}^{\star}\Vert_{2}+\Vert x'-x_{i(x_{t})}^{\star}\Vert_{2}\leq2\sqrt{\frac{C_{3}k(1-\overline{\alpha}_{t})\log T}{\overline{\alpha}_{t}}}+2\varepsilon\leq3\sqrt{\frac{C_{3}k(1-\overline{\alpha}_{t})\log T}{\overline{\alpha}_{t}}}
\]
Hence for any $x,x'\in\mathcal{X}_{t}(x_{t})$ we have
\begin{equation}
\overline{\alpha}_{t}\Vert x-x'\Vert_{2}^{2}\leq9C_{3}k\left(1-\overline{\alpha}_{t}\right)\log T.\label{eq:proof-lemma-set-Xt-dist}
\end{equation}
In addition, for any $x,x'\in\mathcal{X}$, suppose that $x\in\mathcal{B}_{i}$
and $x'\in\mathcal{B}_{j}$. For any $\omega\in\mathcal{G}$, we have
\begin{align}
\left|\omega^{\top}\left(x-x'\right)\right| & =\left|\omega^{\top}\left(x_{i}^{\star}-x_{j}^{\star}\right)\right|+\left|\omega^{\top}\left(x-x_{i}^{\star}\right)\right|+\left|\omega^{\top}\left(x_{j}^{\star}-x'\right)\right|\nonumber \\
 & \overset{\text{(i)}}{\leq}\sqrt{C_{1}k\log T}\Vert x_{i}^{\star}-x_{j}^{\star}\Vert_{2}+\left\Vert x-x_{i}^{\star}\right\Vert _{2}\left\Vert \omega\right\Vert _{2}+\left\Vert x'-x_{j}^{\star}\right\Vert _{2}\left\Vert \omega\right\Vert _{2}\nonumber \\
 & \overset{\text{(ii)}}{\leq}\sqrt{C_{1}k\log T}\Vert x_{i}^{\star}-x_{j}^{\star}\Vert_{2}+2\left(2\sqrt{d}+\sqrt{C_{1}k\log T}\right)\varepsilon\nonumber \\
 & \overset{\text{(iii)}}{\leq}\sqrt{C_{1}k\log T}\Vert x-x'\Vert_{2}+2\sqrt{C_{1}k\log T}\varepsilon+2\left(2\sqrt{d}+\sqrt{C_{1}k\log T}\right)\varepsilon\nonumber \\
 & \leq\sqrt{C_{1}k\log T}\Vert x-x'\Vert_{2}+\big(4\sqrt{d}+4\sqrt{C_{1}k\log T}\big)\varepsilon.\label{eq:proof-lemma-omega-inner}
\end{align}
Here step (i) follows from $\omega\in\mathcal{G}$ and the Cauchy-Schwarz
inequality; steps (ii) and (iii) follows from$\Vert x-x_{i}^{\star}\Vert_{2}\leq\varepsilon$
and $\Vert x'-x_{j}^{\star}\Vert_{2}\leq\varepsilon$, as well as
$\Vert\omega\Vert_{2}\leq\sqrt{d}+\sqrt{C_{1}k\log T}$, which is
a property for $\omega\in\mathcal{G}$.

\subsection{Understanding the conditional density $p_{X_{t}|X_{0}}(\cdot\mymid x_{0})$}

Conditional on $X_{t}=x_{t}$, for any $1\leq i\leq N_{\varepsilon}$
we have
\begin{align}
\mathbb{P}\left(X_{0}\in\mathcal{B}_{i}\mymid X_{t}=x_{t}\right) & =\frac{\mathbb{P}\left(X_{0}\in\mathcal{B}_{i},X_{t}=x_{t}\right)}{p_{X_{t}}\left(x_{t}\right)}=\frac{\mathbb{P}\left(X_{0}\in\mathcal{B}_{i},X_{t}=x_{t}\right)}{\sum_{1\leq j\leq N_{\varepsilon}}\mathbb{P}\left(X_{0}\in\mathcal{B}_{j},X_{t}=x_{t}\right)}\nonumber \\
 & \leq\frac{\mathbb{P}\left(X_{0}\in\mathcal{B}_{i},X_{t}=x_{t}\right)}{\mathbb{P}\left(X_{0}\in\mathcal{B}_{i(x_{t})},X_{t}=x_{t}\right)}=\frac{\mathbb{P}\left(X_{0}\in\mathcal{B}_{i}\right)\mathbb{P}\left(X_{t}=x_{t}\mymid X_{0}\in\mathcal{B}_{i}\right)}{\mathbb{P}\left(X_{0}\in\mathcal{B}_{i(x_{t})}\right)\mathbb{P}\left(X_{t}=x_{t}\mymid X_{0}\in\mathcal{B}_{i(x_{t})}\right)}\nonumber \\
 & \leq\exp\left(C_{1}k\log T\right)\cdot\frac{\mathbb{P}\left(X_{t}=x_{t}\mymid X_{0}\in\mathcal{B}_{i}\right)}{\mathbb{P}\left(X_{t}=x_{t}\mymid X_{0}\in\mathcal{B}_{i(x_{t})}\right)}\cdot\mathbb{P}\left(X_{0}\in\mathcal{B}_{i}\right).\label{eq:proof-lemma-set-1}
\end{align}
Here the last relation follows from $\mathbb{P}(X_{0}\in\mathcal{B}_{i(x_{t})})\geq\exp(-C_{1}k\log T)$
due to $i(x_{t})\in\mathcal{I}$. We have
\begin{align}
\mathbb{P}\left(X_{t}=x_{t}\mymid X_{0}\in\mathcal{B}_{i}\right) & =\frac{\mathbb{P}\left(X_{t}=x_{t},X_{0}\in\mathcal{B}_{i}\right)}{\mathbb{P}\left(X_{0}\in\mathcal{B}_{i}\right)}=\frac{1}{\mathbb{P}\left(X_{0}\in\mathcal{B}_{i}\right)}\int_{\tilde{x}\in\mathcal{B}_{i}}\mathbb{P}\left(X_{t}=x_{t},X_{0}=\tilde{x}\right)\mathrm{d}\tilde{x}\nonumber \\
 & =\frac{1}{\mathbb{P}\left(X_{0}\in\mathcal{B}_{i}\right)}\int_{\tilde{x}\in\mathcal{B}_{i}}\mathbb{P}\left(X_{t}=x_{t}\mymid X_{0}=\tilde{x}\right)\mathbb{P}\left(X_{0}=\tilde{x}\right)\mathrm{d}\tilde{x}\nonumber \\
 & \leq\max_{\tilde{x}\in\mathcal{B}_{i}}\,\mathbb{P}\left(X_{t}=x_{t}\mymid X_{0}=\tilde{x}\right).\label{eq:proof-lemma-set-2}
\end{align}
For any $\tilde{x}\in\mathcal{B}_{i}$, since $X_{t}\mymid X_{0}=\widetilde{x}\sim\mathcal{N}(\sqrt{\overline{\alpha}_{t}}\tilde{x},(1-\overline{\alpha}_{t})I_{d})$,
we have
\begin{align}
\mathbb{P}\left(X_{t}=x_{t}\mymid X_{0}=\tilde{x}\right) & =\left[2\pi\left(1-\overline{\alpha}_{t}\right)\right]^{-d/2}\exp\left(-\frac{\Vert x_{t}-\sqrt{\overline{\alpha}_{t}}\tilde{x}\Vert_{2}^{2}}{2\left(1-\overline{\alpha}_{t}\right)}\right)\nonumber \\
 & \leq\left[2\pi\left(1-\overline{\alpha}_{t}\right)\right]^{-d/2}\exp\left(-\frac{(\Vert x_{t}-\sqrt{\overline{\alpha}_{t}}x_{i}^{\star}\Vert_{2}-\sqrt{\overline{\alpha}_{t}}\varepsilon)^{2}}{2\left(1-\overline{\alpha}_{t}\right)}\right).\label{eq:proof-lemma-set-3}
\end{align}
Taking (\ref{eq:proof-lemma-set-2}) and (\ref{eq:proof-lemma-set-3})
collectively to achieve
\begin{align}
\mathbb{P}\left(X_{t}=x_{t}\mymid X_{0}\in\mathcal{B}_{i}\right) & \leq\left[2\pi\left(1-\overline{\alpha}_{t}\right)\right]^{-d/2}\exp\left(-\frac{(\Vert x_{t}-\sqrt{\overline{\alpha}_{t}}x_{i}^{\star}\Vert_{2}-\sqrt{\overline{\alpha}_{t}}\varepsilon)^{2}}{2\left(1-\overline{\alpha}_{t}\right)}\right).\label{eq:proof-lemma-set-4}
\end{align}
By similar argument in (\ref{eq:proof-lemma-set-2}), (\ref{eq:proof-lemma-set-3})
and (\ref{eq:proof-lemma-set-4}), we can show that
\begin{equation}
\mathbb{P}\left(X_{t}=x_{t}\mymid X_{0}\in\mathcal{B}_{i(x_{t})}\right)\geq\left[2\pi\left(1-\overline{\alpha}_{t}\right)\right]^{-d/2}\exp\Bigg(-\frac{(\Vert x_{t}-\sqrt{\overline{\alpha}_{t}}x_{i(x_{t})}^{\star}\Vert_{2}+\sqrt{\overline{\alpha}_{t}}\varepsilon)^{2}}{2\left(1-\overline{\alpha}_{t}\right)}\Bigg).\label{eq:proof-lemma-set-5}
\end{equation}
Combine (\ref{eq:proof-lemma-set-4}) and (\ref{eq:proof-lemma-set-5})
to achieve
\begin{align}
 & \frac{\mathbb{P}\left(X_{t}=x_{t}\mymid X_{0}\in\mathcal{B}_{i}\right)}{\mathbb{P}\left(X_{t}=x_{t}\mymid X_{0}\in\mathcal{B}_{i(x_{t})}\right)}\leq\exp\left[-\frac{(\Vert x_{t}-\sqrt{\overline{\alpha}_{t}}x_{i}^{\star}\Vert_{2}-\sqrt{\overline{\alpha}_{t}}\varepsilon)^{2}}{2\left(1-\overline{\alpha}_{t}\right)}+\frac{(\Vert x_{t}-\sqrt{\overline{\alpha}_{t}}x_{i(x_{t})}^{\star}\Vert_{2}+\sqrt{\overline{\alpha}_{t}}\varepsilon)^{2}}{2\left(1-\overline{\alpha}_{t}\right)}\right]\nonumber \\
 & \quad\leq\exp\left[-\frac{\Vert x_{t}-\sqrt{\overline{\alpha}_{t}}x_{i}^{\star}\Vert_{2}^{2}-\Vert x_{t}-\sqrt{\overline{\alpha}_{t}}x_{i(x_{t})}^{\star}\Vert_{2}^{2}-2\sqrt{\overline{\alpha}_{t}}\varepsilon(\Vert x_{t}-\sqrt{\overline{\alpha}_{t}}x_{i}^{\star}\Vert_{2}+\Vert x_{t}-\sqrt{\overline{\alpha}_{t}}x_{i(x_{t})}^{\star}\Vert_{2})}{2\left(1-\overline{\alpha}_{t}\right)}\right].\label{eq:proof-lemma-set-6}
\end{align}
Next, we will discuss the implication of the above analysis for any $i\notin\mathcal{I}(x_{t};C_{3})$, i.e., $\mathcal{B}_{i}\subseteq\mathcal{Y}_{t}(x_{t})$.

For any $i\notin\mathcal{I}(x_{t};C_{3})$, we have
\begin{align}
 & \Vert x_{t}-\sqrt{\overline{\alpha}_{t}}x_{i}^{\star}\Vert_{2}^{2}-\Vert x_{t}-\sqrt{\overline{\alpha}_{t}}x_{i(x_{t})}^{\star}\Vert_{2}^{2}=\overline{\alpha}_{t}\Vert x_{i(x_{t})}^{\star}-x_{i}^{\star}\Vert_{2}^{2}+2\sqrt{\overline{\alpha}_{t}}(x_{i(x_{t})}^{\star}-x_{i}^{\star})^{\top}(x_{t}-\sqrt{\overline{\alpha}_{t}}x_{i(x_{t})}^{\star})\nonumber \\
 & \quad\overset{\text{(i)}}{=}\overline{\alpha}_{t}\Vert x_{i(x_{t})}^{\star}-x_{i}^{\star}\Vert_{2}^{2}+2\sqrt{\overline{\alpha}_{t}\left(1-\overline{\alpha}_{t}\right)}(x_{i(x_{t})}^{\star}-x_{i}^{\star})^{\top}\omega+2\overline{\alpha}_{t}\sqrt{1-\overline{\alpha}_{t}}(x_{i(x_{t})}^{\star}-x_{i}^{\star})^{\top}(x_{0}(x_{t})-x_{i(x_{t})}^{\star})\nonumber \\
 & \quad\overset{\text{(ii)}}{\geq}\overline{\alpha}_{t}\Vert x_{i(x_{t})}^{\star}-x_{i}^{\star}\Vert_{2}^{2}-2\sqrt{C_{1}\overline{\alpha}_{t}\left(1-\overline{\alpha}_{t}\right)}\sqrt{k\log T}\Vert x_{i(x_{t})}^{\star}-x_{i}^{\star}\Vert_{2}-2\overline{\alpha}_{t}\sqrt{1-\overline{\alpha}_{t}}\varepsilon\Vert x_{i(x_{t})}^{\star}-x_{i}^{\star}\Vert_{2}\nonumber \\
 & \quad\overset{\text{(iii)}}{\geq}\frac{1}{2}\overline{\alpha}_{t}\Vert x_{i(x_{t})}^{\star}-x_{i}^{\star}\Vert_{2}^{2}.\label{eq:proof-lemma-set-7'}
\end{align}
Here step (i) uses the decomposition (\ref{eq:proof-lemma-set-decom});
step (ii) follows from $\omega\in\mathcal{G}$, Cauchy-Schwarz inequality
and the fact that $\Vert x_{0}(x_{t})-x_{i(x_{t})}^{\star}\Vert_{2}\leq\varepsilon$;
while the correctness of step (iii) is equivalent to
\[
\sqrt{\overline{\alpha}_{t}}\Vert x_{i(x_{t})}^{\star}-x_{i}^{\star}\Vert_{2}\geq4\sqrt{C_{1}\left(1-\overline{\alpha}_{t}\right)}\sqrt{k\log T}+4\sqrt{\overline{\alpha}_{t}\left(1-\overline{\alpha}_{t}\right)}\varepsilon,
\]
which follows from $i\notin\mathcal{I}(x_{t};C_{3})$, the assumptions
that $C_{3}\gg C_{1}$, and (\ref{eq:eps-condition}). In addition,
we also have
\begin{align}
 & \Vert x_{t}-\sqrt{\overline{\alpha}_{t}}x_{i}^{\star}\Vert_{2}+\Vert x_{t}-\sqrt{\overline{\alpha}_{t}}x_{i(x_{0})}^{\star}\Vert_{2}\overset{\text{(i)}}{\leq}\sqrt{\overline{\alpha}_{t}}\Vert x_{0}(x_{t})-x_{i}^{\star}\Vert_{2}+\sqrt{\overline{\alpha}_{t}}\Vert x_{0}(x_{t})-x_{i(x_{t})}^{\star}\Vert_{2}+2\sqrt{1-\overline{\alpha}_{t}}\Vert\omega\Vert_{2}\nonumber \\
 & \qquad\overset{\text{(ii)}}{\leq}\sqrt{\overline{\alpha}_{t}}\Vert x_{i(x_{t})}^{\star}-x_{i}^{\star}\Vert_{2}+2\sqrt{\overline{\alpha}_{t}}\Vert x_{0}(x_{t})-x_{i(x_{t})}^{\star}\Vert_{2}+2\sqrt{1-\overline{\alpha}_{t}}\left(2\sqrt{d}+\sqrt{C_{1}k\log T}\right)\nonumber \\
 & \qquad\overset{\text{(iii)}}{\leq}\sqrt{\overline{\alpha}_{t}}\Vert x_{i(x_{t})}^{\star}-x_{i}^{\star}\Vert_{2}+3\sqrt{1-\overline{\alpha}_{t}} (2\sqrt{d}+\sqrt{C_{1}k\log T} ).\label{eq:proof-lemma-set-8'}
\end{align}
Here step (i) follows from the decomposition (\ref{eq:proof-lemma-set-decom});
step (ii) utilizes the triangle inequality; whereas step (iii) follows
from $\Vert x_{0}(x_{t})-x_{i(x_{t})}^{\star}\Vert_{2}\leq\varepsilon$
and the condition (\ref{eq:eps-condition}). We can substitute the bounds (\ref{eq:proof-lemma-set-7'}) and (\ref{eq:proof-lemma-set-8'})
into (\ref{eq:proof-lemma-set-6}) to get
\begin{align}
\frac{\mathbb{P}(X_{t}=x_{t}\mymid X_{0}\in\mathcal{B}_{i})}{\mathbb{P}(X_{t}=x_{t}\mymid X_{0}\in\mathcal{B}_{i(x_{t})})} & \leq\exp\left(-\frac{\overline{\alpha}_{t}}{8\left(1-\overline{\alpha}_{t}\right)}\Vert x_{i(x_{t})}^{\star}-x_{i}^{\star}\Vert_{2}^{2}\right),\label{eq:proof-lemma-set-9'}
\end{align}
provided that (\ref{eq:eps-condition}) holds and $C_{3}$ is sufficiently
large.
Since $i\notin\mathcal{I}(x_{t};C_{3})$, we know that $\overline{\alpha}_{t}\Vert x_{i(x_{t})}^{\star}-x_{i}^{\star}\Vert_{2}^{2}>C_{3}k(1-\overline{\alpha}_{t})\log T$,
hence when $C_{3}\gg C_{1}$, we learn from (\ref{eq:proof-lemma-set-1})
and (\ref{eq:proof-lemma-set-9'}) that
\begin{align}
\mathbb{P}\left(X_{0}\in\mathcal{B}_{i}\mymid X_{t}=x_{t}\right) & \leq\exp\left(C_{1}k\log T-\frac{\overline{\alpha}_{t}}{8\left(1-\overline{\alpha}_{t}\right)}\Vert x_{i(x_{t})}^{\star}-x_{i}^{\star}\Vert_{2}^{2}\right)\mathbb{P}\left(X_{0}\in\mathcal{B}_{i}\right)\nonumber \\
 & \leq\exp\left(-\frac{\overline{\alpha}_{t}}{16\left(1-\overline{\alpha}_{t}\right)}\Vert x_{i(x_{t})}^{\star}-x_{i}^{\star}\Vert_{2}^{2}\right)\mathbb{P}\left(X_{0}\in\mathcal{B}_{i}\right).\label{eq:proof-lemma-set-10'}
\end{align}

\subsection{Proof of Lemma \ref{lemma:At-SDE} \label{subsec:proof-lemma-At-SDE}}

It is straightforward to check that
\[
\mathbb{P}\left(\left(X_{t},X_{t-1}\right)\in\mathcal{A}_{t}\right)\overset{\text{(i)}}{=}\mathbb{P}\left(X_{t}\in\mathcal{T}_{t},W_{t}\in\mathcal{G}\right)\overset{\text{(ii)}}{\leq}\mathbb{P}\left(X_{0}\in\cup_{i\in\mathcal{I}}\mathcal{B}_{i},W_{t}\in\mathcal{G},\overline{W}_{t}\in\mathcal{G}\right),
\]
where step (i) follows from the update rule (\ref{eq:forward-update}),
and step (ii) follows from the relation (\ref{eq:forward-formula}).
Therefore we have
\begin{equation}
\mathbb{P}\left(\left(X_{t},X_{t-1}\right)\notin\mathcal{A}_{t}\right)\leq\mathbb{P}\left(X_{0}\notin\cup_{i\in\mathcal{I}}\mathcal{B}_{i}\right)+\mathbb{P}\left(W_{t}\notin\mathcal{G}\right)+\mathbb{P}\left(\overline{W}_{t}\notin\mathcal{G}\right).\label{eq:proof-At-1}
\end{equation}
By definition of the set $\mathcal{I}$, we have
\begin{equation}
\mathbb{P}\left(X_{0}\notin\cup_{i\in\mathcal{I}}\mathcal{B}_{i}\right)\leq N_{\varepsilon}\exp\left(-C_{1}k\log T\right)\leq\exp\left(C_{\mathsf{cover}}k\log T-C_{1}k\log T\right)\leq\frac{1}{3}\exp\left(-\frac{C_{1}}{4}k\log T\right)\label{eq:proof-At-2}
\end{equation}
as long as $C_{1}\gg C_{\mathsf{cover}}$. In addition, since $W_{t},\overline{W}_{t}\sim\mathcal{N}(0,I_{d})$,
by the definition of $\mathcal{G}$ we know that
\begin{align}
\mathbb{P}\left(W_{t}\notin\mathcal{G}\right) & \leq\mathbb{P}\left(\Vert W_{t}\Vert_{2}>\sqrt{d}+\sqrt{C_{1}k\log T}\right)+\sum_{i=1}^{N_{\varepsilon}}\sum_{j=1}^{N_{\varepsilon}}\mathbb{P}\left(\vert(x_{i}^{\star}-x_{j}^{\star})^{\top}W_{t}\vert>\sqrt{C_{1}k\log T}\Vert x_{i}^{\star}-x_{j}^{\star}\Vert_{2}\right)\nonumber \\
 & \overset{\text{(i)}}{\leq}\left(N_{\varepsilon}^{2}+1\right)\exp\left(-\frac{C_{1}}{2}k\log T\right)\leq\left(\exp\left(2C_{\mathsf{cover}}k\log T\right)+1\right)\exp\left(-\frac{C_{1}}{2}k\log T\right)\nonumber \\
 & \overset{\text{(ii)}}{\leq}\frac{1}{3}\exp\left(-\frac{C_{1}}{4}k\log T\right)\label{eq:proof-At-3}
\end{align}
Here step (i) follows from concentration bounds for Gaussian and chi-square
variables (see Lemma~\ref{lemma:concentration}); while step (ii)
holds as long as $C_{1}\gg C_{\mathsf{cover}}$. The same bound also
holds for $\mathbb{P}(\overline{W}_{t}\notin\mathcal{G})$. Taking
(\ref{eq:proof-At-1}), (\ref{eq:proof-At-2}) and (\ref{eq:proof-At-3})
collectively yields
\[
\mathbb{P}\left(\left(X_{t},X_{t-1}\right)\notin\mathcal{A}_{t}\right)\leq\exp\left(-\frac{C_{1}}{4}k\log T\right)
\]
as claimed. 

\subsection{Proof of Lemma \ref{lemma:cond-density} \label{subsec:proof-lemma-cond-density}}

For any deterministic pairs $(x_{t},x_{t-1})$, we have
\begin{align}
p_{X_{t-1}|X_{t}}\left(x_{t-1}\mymid x_{t}\right) & =\frac{1}{p_{X_{t}}\left(x_{t}\right)}p_{X_{t-1},X_{t}}\left(x_{t-1},x_{t}\right)=\frac{p_{X_{t-1}}\left(x_{t-1}\right)}{p_{X_{t}}\left(x_{t}\right)}p_{X_{t}|X_{t-1}}\left(x_{t}\mymid x_{t-1}\right).\label{eq:proof-sde-1}
\end{align}
Recall that $X_{t}\mymid X_{t-1}=x_{t-1}\sim\mathcal{N}(\sqrt{\alpha_{t}}x_{t-1},(1-\alpha_{t})I_{d})$,
therefore we have
\begin{align}
p_{X_{t}|X_{t-1}}\left(x_{t}\mymid x_{t-1}\right) & =\left[2\pi\left(1-\alpha_{t}\right)\right]^{-d/2}\exp\left(-\frac{1}{2\left(1-\alpha_{t}\right)}\Vert x_{t}-\sqrt{\alpha_{t}}x_{t-1}\Vert_{2}^{2}\right).\label{eq:proof-sde-2}
\end{align}
Next, we analyze the density ratio $p_{X_{t-1}}(x_{t-1})/p_{X_{t}}(x_{t})$.
It would be easier to do a change of variable 
\begin{equation}
p_{X_{t-1}}\left(x_{t-1}\right)=\alpha_{t}^{d/2}p_{\sqrt{\alpha_{t}}X_{t-1}}\left(\sqrt{\alpha_{t}}x_{t-1}\right).\label{eq:proof-sde-3}
\end{equation}
Since $\sqrt{\alpha_{t}}X_{t-1}\mymid X_{0}=x_{0}\sim\mathcal{N}(\sqrt{\overline{\alpha}_{t}}x_{0},(\alpha_{t}-\overline{\alpha}_{t})I_{d})$,
we can write
\begin{align}
\frac{p_{\sqrt{\alpha_{t}}X_{t-1}}\left(\sqrt{\alpha_{t}}x_{t-1}\right)}{p_{X_{t}}\left(x_{t}\right)} & =\frac{1}{p_{X_{t}}\left(x_{t}\right)}\int_{x_{0}}p_{X_{0}}\left(x_{0}\right)p_{\sqrt{\alpha_{t}}X_{t-1}|X_{0}}\left(\sqrt{\alpha_{t}}x_{t-1}\mymid x_{0}\right)\mathrm{d}x_{0}\nonumber \\
 & =\frac{1}{p_{X_{t}}\left(x_{t}\right)}\int_{x_{0}}p_{X_{0}}\left(x_{0}\right)\left[2\pi\left(\alpha_{t}-\overline{\alpha}_{t}\right)\right]^{-d/2}\exp\left(-\frac{\Vert\sqrt{\alpha_{t}}x_{t-1}-\sqrt{\overline{\alpha}_{t}}x_{0}\Vert_{2}^{2}}{2(\alpha_{t}-\overline{\alpha}_{t})}\right)\mathrm{d}x_{0}.\label{eq:proof-sde-4}
\end{align}
We hope to connect the above quantity with the conditional density
\begin{align}
p_{X_{0}|X_{t}}\left(x_{0}\mymid x_{t}\right) & =\frac{p_{X_{0},X_{t}}\left(x_{0},x_{t}\right)}{p_{X_{t}}\left(x_{t}\right)}=\frac{p_{X_{0}}\left(x_{0}\right)}{p_{X_{t}}\left(x_{t}\right)}p_{X_{t}|X_{0}}\left(x_{t}\mymid x_{0}\right)\nonumber \\
 & =\frac{p_{X_{0}}\left(x_{0}\right)}{p_{X_{t}}\left(x_{t}\right)}\frac{1}{\left[2\pi(1-\overline{\alpha}_{t})\right]^{d/2}}\exp\left(-\frac{\Vert x_{t}-\sqrt{\overline{\alpha}_{t}}x_{0}\Vert_{2}^{2}}{2\left(1-\overline{\alpha}_{t}\right)}\right).\label{eq:proof-sde-5}
\end{align}
Towards this, we can deduce that
\begin{align}
 & \frac{p_{\sqrt{\alpha_{t}}X_{t-1}}\left(\sqrt{\alpha_{t}}x_{t-1}\right)}{p_{X_{t}}\left(x_{t}\right)}\overset{\text{(i)}}{=}\left(\frac{1-\overline{\alpha}_{t}}{\alpha_{t}-\overline{\alpha}_{t}}\right)^{d/2}\int_{x_{0}}\frac{p_{X_{0}}\left(x_{0}\right)}{p_{X_{t}}\left(x_{t}\right)}\left[2\pi(1-\overline{\alpha}_{t})\right]^{-d/2}\exp\left(-\frac{\Vert\sqrt{\alpha_{t}}x_{t-1}-\sqrt{\overline{\alpha}_{t}}x_{0}\Vert_{2}^{2}}{2(\alpha_{t}-\overline{\alpha}_{t})}\right)\mathrm{d}x_{0}\nonumber \\
 & \quad\overset{\text{(ii)}}{=}\left(\frac{1-\overline{\alpha}_{t}}{\alpha_{t}-\overline{\alpha}_{t}}\right)^{d/2}\int_{x_{0}}p_{X_{0}|X_{t}}\left(x_{0}\mymid x_{t}\right)\exp\left(\frac{\Vert x_{t}-\sqrt{\overline{\alpha}_{t}}x_{0}\Vert_{2}^{2}}{2\left(1-\overline{\alpha}_{t}\right)}-\frac{\Vert\sqrt{\alpha_{t}}x_{t-1}-\sqrt{\overline{\alpha}_{t}}x_{0}\Vert_{2}^{2}}{2(\alpha_{t}-\overline{\alpha}_{t})}\right)\mathrm{d}x_{0},\label{eq:proof-sde-5.5}
\end{align}
where step (i) follows from (\ref{eq:proof-sde-4}) and step (ii)
utilizes (\ref{eq:proof-sde-5}). The terms in the exponent can be
rearranged into
\begin{align*}
 & \frac{\Vert x_{t}-\sqrt{\overline{\alpha}_{t}}x_{0}\Vert_{2}^{2}}{2\left(1-\overline{\alpha}_{t}\right)}-\frac{\Vert\sqrt{\alpha_{t}}x_{t-1}-\sqrt{\overline{\alpha}_{t}}x_{0}\Vert_{2}^{2}}{2\left(\alpha_{t}-\overline{\alpha}_{t}\right)}=\frac{\Vert x_{t}-\sqrt{\overline{\alpha}_{t}}x_{0}\Vert_{2}^{2}-\Vert\sqrt{\alpha_{t}}x_{t-1}-\sqrt{\overline{\alpha}_{t}}x_{0}\Vert_{2}^{2}}{2\left(\alpha_{t}-\overline{\alpha}_{t}\right)}-\frac{\left(1-\alpha_{t}\right)\Vert x_{t}-\sqrt{\overline{\alpha}_{t}}x_{0}\Vert_{2}^{2}}{2\left(\alpha_{t}-\overline{\alpha}_{t}\right)\left(1-\overline{\alpha}_{t}\right)}\\
 & \quad=-\frac{\Vert\sqrt{\alpha_{t}}x_{t-1}-x_{t}\Vert_{2}^{2}+2\left(\sqrt{\alpha_{t}}x_{t-1}-x_{t}\right)^{\top}\left(x_{t}-\sqrt{\overline{\alpha}_{t}}x_{0}\right)}{2\left(\alpha_{t}-\overline{\alpha}_{t}\right)}-\frac{\left(1-\alpha_{t}\right)\Vert x_{t}-\sqrt{\overline{\alpha}_{t}}x_{0}\Vert_{2}^{2}}{2\left(\alpha_{t}-\overline{\alpha}_{t}\right)\left(1-\overline{\alpha}_{t}\right)}\\
 & \quad=-\frac{\Vert\sqrt{\alpha_{t}}x_{t-1}-x_{t}\Vert_{2}^{2}+2\left(\sqrt{\alpha_{t}}x_{t-1}-x_{t}\right)^{\top}\left(x_{t}-\sqrt{\overline{\alpha}_{t}}\widehat{x}_{0}\right)}{2\left(\alpha_{t}-\overline{\alpha}_{t}\right)}-\frac{\left(1-\alpha_{t}\right)\Vert x_{t}-\sqrt{\overline{\alpha}_{t}}\widehat{x}_{0}\Vert_{2}^{2}}{2\left(\alpha_{t}-\overline{\alpha}_{t}\right)\left(1-\overline{\alpha}_{t}\right)}+\Delta_{x_{t},x_{t-1}}\left(x_{0}\right)
\end{align*}
where we define
\[
\widehat{x}_{0}\coloneqq\mathbb{E}\left[X_{0}\mymid X_{t}=x_{t}\right]=\int_{x_{0}}x_{0}p_{X_{0}|X_{t}}\left(x_{0}\mymid x_{t}\right)\mathrm{d}x_{0},
\]
and
\begin{align*}
\Delta_{x_{t},x_{t-1}}\left(x_{0}\right) & \coloneqq-\frac{\sqrt{\overline{\alpha}_{t}}}{\alpha_{t}-\overline{\alpha}_{t}}\left(\sqrt{\alpha_{t}}x_{t-1}-x_{t}\right)^{\top}\left(\widehat{x}_{0}-x_{0}\right)-\frac{\left(1-\alpha_{t}\right)\sqrt{\overline{\alpha}_{t}}}{\left(\alpha_{t}-\overline{\alpha}_{t}\right)\left(1-\overline{\alpha}_{t}\right)}\left(x_{t}-\sqrt{\overline{\alpha}_{t}}\widehat{x}_{0}\right)^{\top}\left(\widehat{x}_{0}-x_{0}\right)\\
 & \qquad-\frac{\left(1-\alpha_{t}\right)\overline{\alpha}_{t}}{2\left(\alpha_{t}-\overline{\alpha}_{t}\right)\left(1-\overline{\alpha}_{t}\right)}\Vert\widehat{x}_{0}-x_{0}\Vert_{2}^{2}.
\end{align*}
Substituting the above relation into (\ref{eq:proof-sde-5.5}) yields
\begin{align}
\frac{p_{\sqrt{\alpha_{t}}X_{t-1}}\left(\sqrt{\alpha_{t}}x_{t-1}\right)}{p_{X_{t}}\left(x_{t}\right)} & =\left(\frac{1-\overline{\alpha}_{t}}{\alpha_{t}-\overline{\alpha}_{t}}\right)^{d/2}\exp\left[-\frac{\Vert\sqrt{\alpha_{t}}x_{t-1}-x_{t}\Vert_{2}^{2}}{2\left(\alpha_{t}-\overline{\alpha}_{t}\right)}\right]\nonumber \\
 & \qquad\quad\cdot\exp\left[-\frac{\left(\sqrt{\alpha_{t}}x_{t-1}-x_{t}\right)^{\top}\left(x_{t}-\sqrt{\overline{\alpha}_{t}}\widehat{x}_{0}\right)}{\alpha_{t}-\overline{\alpha}_{t}}-\frac{\left(1-\alpha_{t}\right)\Vert x_{t}-\sqrt{\overline{\alpha}_{t}}\widehat{x}_{0}\Vert_{2}^{2}}{2\left(\alpha_{t}-\overline{\alpha}_{t}\right)\left(1-\overline{\alpha}_{t}\right)}\right]\nonumber \\
 & \qquad\quad\cdot\int_{x_{0}}p_{X_{0}|X_{t}}\left(x_{0}\mymid x_{t}\right)\exp\left(\Delta_{x_{t},x_{t-1}}\left(x_{0}\right)\right)\mathrm{d}x_{0}.\label{eq:proof-sde-6}
\end{align}
Therefore we have
\begin{align}
 & p_{X_{t-1}|X_{t}}\left(x_{t-1}\mymid x_{t}\right)\overset{\text{(i)}}{=}\alpha_{t}^{d/2}\frac{p_{\sqrt{\alpha_{t}}X_{t-1}}\left(\sqrt{\alpha_{t}}x_{t-1}\right)}{p_{X_{t}}\left(x_{t}\right)}p_{X_{t}|X_{t-1}}\left(x_{t}\mymid x_{t-1}\right)\label{eq:proof-sde-7}\\
 & \quad\overset{\text{(ii)}}{=}\alpha_{t}^{d/2}\left(\frac{1-\overline{\alpha}_{t}}{\alpha_{t}-\overline{\alpha}_{t}}\right)^{d/2}\exp\left(-\frac{\left(\sqrt{\alpha_{t}}x_{t-1}-x_{t}\right)^{\top}\left(x_{t}-\sqrt{\overline{\alpha}_{t}}\widehat{x}_{0}\right)}{\alpha_{t}-\overline{\alpha}_{t}}-\frac{\left(1-\alpha_{t}\right)\Vert x_{t}-\sqrt{\overline{\alpha}_{t}}\widehat{x}_{0}\Vert_{2}^{2}}{2\left(\alpha_{t}-\overline{\alpha}_{t}\right)\left(1-\overline{\alpha}_{t}\right)}\right)\nonumber \\
 & \quad\quad\cdot\left[2\pi\left(1-\alpha_{t}\right)\right]^{-d/2}\exp\left(-\frac{\left(1-\overline{\alpha}_{t}\right)\left\Vert x_{t}-\sqrt{\alpha_{t}}x_{t-1}\right\Vert _{2}^{2}}{2\left(1-\alpha_{t}\right)\left(\alpha_{t}-\overline{\alpha}_{t}\right)}\right)\cdot\int_{x_{0}}p_{X_{0}|X_{t}}\left(x_{0}\mymid x_{t}\right)\exp\left(\Delta_{x_{t},x_{t-1}}\left(x_{0}\right)\right)\mathrm{d}x_{0}\nonumber \\
 & \quad\overset{\text{(iii)}}{=}\frac{\alpha_{t}^{d/2}}{\left(2\pi\sigma_{t}^{\star2}\right)^{d/2}}\exp\left(-\frac{\Vert\sqrt{\alpha_{t}}x_{t-1}-x_{t}-\eta_{t}^{\star}s_{t}^{\star}\left(x_{t}\right)\Vert_{2}^{2}}{2\sigma_{t}^{\star2}}\right)\cdot\int_{x_{0}}p_{X_{0}|X_{t}}\left(x_{0}\mymid x_{t}\right)\exp\left(\Delta_{x_{t},x_{t-1}}\left(x_{0}\right)\right)\mathrm{d}x_{0}.\nonumber 
\end{align}
Here step (i) follows from (\ref{eq:proof-sde-1}) and (\ref{eq:proof-sde-3});
step (ii) follows from (\ref{eq:proof-sde-2}) and (\ref{eq:proof-sde-6});
whereas step (iii) follows from the definition of $\eta_{t}^{\star}$
and $\sigma_{t}^{\star}$ (cf.~(\ref{eq:defn-step-size})) as well
as the fact that
\[
s_{t}^{\star}\left(x_{t}\right)=-\frac{1}{1-\overline{\alpha}_{t}}\int_{x_{0}}p_{X_{0}|X_{t}}\left(x_{0}\mymid x_{t}\right)\left(x-\sqrt{\overline{\alpha}_{t}}x_{0}\right)\mathrm{d}x_{0}=-\frac{1}{1-\overline{\alpha}_{t}}\left(x_{t}-\sqrt{\overline{\alpha}_{t}}\widehat{x}_{0}\right).
\]

\subsection{Proof of Lemma \ref{lemma:Delta-SDE} \label{subsec:proof-lemma-Delta-SDE}}

For any $(x_{t},x_{t-1})\in\mathcal{A}_{t}$, we know that $\omega'\coloneqq(x_{t}-\sqrt{\alpha_{t}}x_{t-1})/\sqrt{1-\alpha_{t}}\in\mathcal{G}$.
We will upper bound the integral with two terms
\begin{align*}
\int_{x_{0}}p_{X_{0}|X_{t}}\left(x_{0}\mymid x_{t}\right)\exp\left(\Delta\left(x_{t},x_{t-1},x_{0}\right)\right)\mathrm{d}x_{0} & =\underbrace{\int_{\mathcal{X}_{t}(x_{t})}p_{X_{0}|X_{t}}\left(x_{0}\mymid x_{t}\right)\exp\left(\Delta\left(x_{t},x_{t-1},x_{0}\right)\right)\mathrm{d}x_{0}}_{\eqqcolon I_{1}}\\
 & \quad+\underbrace{\int_{\mathcal{Y}_{t}(x_{t})}p_{X_{0}|X_{t}}\left(x_{0}\mymid x_{t}\right)\exp\left(\Delta\left(x_{t},x_{t-1},x_{0}\right)\right)\mathrm{d}x_{0}}_{\eqqcolon I_{2}},
\end{align*}
where we recall that
\begin{align*}
\Delta\left(x_{t},x_{t-1},x_{0}\right) & =\underbrace{\frac{\sqrt{\overline{\alpha}_{t}\left(1-\alpha_{t}\right)}}{\alpha_{t}-\overline{\alpha}_{t}}\left(\widehat{x}_{0}-x_{0}\right)^{\top}\omega'}_{\eqqcolon\Delta_{1}(x_{0})}-\underbrace{\frac{\left(1-\alpha_{t}\right)\sqrt{\overline{\alpha}_{t}}}{\left(\alpha_{t}-\overline{\alpha}_{t}\right)\sqrt{1-\overline{\alpha}_{t}}}\left(\widehat{x}_{0}-x_{0}\right)^{\top}\omega}_{\eqqcolon\Delta_{2}(x_{0})}\\
 & \qquad-\underbrace{\frac{\left(1-\alpha_{t}\right)\overline{\alpha}_{t}}{\left(\alpha_{t}-\overline{\alpha}_{t}\right)\left(1-\overline{\alpha}_{t}\right)}\left(x_{0}(x_{t})-\widehat{x}_{0}\right)^{\top}\left(\widehat{x}_{0}-x_{0}\right)}_{\eqqcolon\Delta_{3}(x_{0})}-\underbrace{\frac{\left(1-\alpha_{t}\right)\overline{\alpha}_{t}}{2\left(\alpha_{t}-\overline{\alpha}_{t}\right)\left(1-\overline{\alpha}_{t}\right)}\left\Vert \widehat{x}_{0}-x_{0}\right\Vert _{2}^{2}}_{\eqqcolon\Delta_{4}(x_{0})}.
\end{align*}
In what follows, we will use $\Delta(x_{0})$ instead of $\Delta(x_{t},x_{t-1},x_{0})$
when there is no confusion. Since $(x_{t},x_{t-1})\in\mathcal{A}_{t}$,
we know that $\omega'\coloneqq(x_{t}-\sqrt{\alpha_{t}}x_{t-1})/\sqrt{1-\alpha_{t}}\in\mathcal{G}$.
We decompose $\widehat{x}_{0}$ into
\begin{align}
\widehat{x}_{0} & =\int_{x_{0}}x_{0}'p_{X_{0}|X_{t}}\left(x_{0}'\mymid x_{t}\right)\mathrm{d}x_{0}'\nonumber \\
 & =\underbrace{x_{i(x_{t})}^{\star}+\int_{\mathcal{X}_{t}(x_{t})}(x_{0}'-x_{i(x_{t})}^{\star})p_{X_{0}|X_{t}}\left(x_{0}'\mymid x_{t}\right)\mathrm{d}x_{0}'}_{\eqqcolon\overline{x}_{0}}+\underbrace{\int_{\mathcal{Y}_{t}(x_{t})}(x_{0}'-x_{i(x_{t})}^{\star})p_{X_{0}|X_{t}}\left(x_{0}'\mymid x_{t}\right)\mathrm{d}x_{0}'}_{\eqqcolon\delta}.\label{eq:x0-hat-decom}
\end{align}
Since $\mathcal{X}_{t}(x_{t})$ is a ball in $\mathbb{R}^{d}$ centered
at $x_{i(x_{t})}^{\star}$, it is straightforward to check that $\overline{x}_{0}\in\mathcal{X}_{t}(x_{t})$.
We also have
\begin{align*}
\Vert\delta\Vert_{2} & \leq\int_{\mathcal{Y}_{t}(x_{t})}\Vert x_{0}'-x_{i(x_{t})}^{\star}\Vert_{2}p_{X_{0}|X_{t}}\left(x_{0}\mymid x_{t}\right)\mathrm{d}x_{0}'\\
 & \overset{\text{(i)}}{\leq}\sum_{i\notin\mathcal{I}(x_{t};C_{3})}\left(\Vert x_{i}^{\star}-x_{i(x_{t})}^{\star}\Vert_{2}+\varepsilon\right)\mathbb{P}\left(X_{0}\in\mathcal{B}_{i}\mymid X_{t}=x_{t}\right)\\
 & \overset{\text{(ii)}}{\leq}\sum_{i\notin\mathcal{I}(x_{t};C_{3})}\left(\Vert x_{i}^{\star}-x_{i(x_{t})}^{\star}\Vert_{2}+\varepsilon\right)\exp\left(-\frac{\overline{\alpha}_{t}}{16\left(1-\overline{\alpha}_{t}\right)}\Vert x_{i(x_{t})}^{\star}-x_{i}^{\star}\Vert_{2}^{2}\right)\mathbb{P}\left(X_{0}\in\mathcal{B}_{i}\right).
\end{align*}
Here step (i) holds since for any $i\notin\mathcal{I}(x_{t};C_{3})$
and $x_{0}'\in\mathcal{B}_{i}$, 
\[
\Vert x_{0}'-x_{i(x_{t})}^{\star}\Vert_{2}\leq\Vert x_{i}^{\star}-x_{i(x_{t})}^{\star}\Vert_{2}+\Vert x_{0}'-x_{i}^{\star}\Vert_{2}\leq\Vert x_{i}^{\star}-x_{i(x_{t})}^{\star}\Vert_{2}+\varepsilon;
\]
while step (ii) follows from (\ref{eq:proof-lemma-set-10'}). For
any $i\notin\mathcal{I}(x_{t};C_{3})$, we know that $\overline{\alpha}_{t}\Vert x_{i(x_{t})}^{\star}-x_{i}^{\star}\Vert_{2}^{2}>C_{3}k(1-\overline{\alpha}_{t})\log T$,
hence we can check that 
\begin{align*}
\big(\Vert x_{i}^{\star}-x_{i(x_{t})}^{\star}\Vert_{2}+\varepsilon\big)\exp\left(-\frac{\overline{\alpha}_{t}}{16\left(1-\overline{\alpha}_{t}\right)}\Vert x_{i(x_{t})}^{\star}-x_{i}^{\star}\Vert_{2}^{2}\right) & \leq\Bigg(\sqrt{\frac{C_{3}k(1-\overline{\alpha}_{t})\log T}{\overline{\alpha}_{t}}}+\varepsilon\Bigg)\exp\left(-\frac{C_{3}k\log T}{16}\right)\\
 & \leq\sqrt{\frac{1-\overline{\alpha}_{t}}{\overline{\alpha}_{t}}}\exp\left(-\frac{C_{3}k\log T}{32}\right)
\end{align*}
as long as $C_{3}$ is sufficiently large and the condition (\ref{eq:eps-condition})
holds. Therefore
\begin{align}
\Vert\delta\Vert_{2} & \leq\sum_{i\notin\mathcal{I}(x_{t};C_{3})}\sqrt{\frac{1-\overline{\alpha}_{t}}{\overline{\alpha}_{t}}}\exp\left(-\frac{C_{3}k\log T}{32}\right)\mathbb{P}\left(X_{0}\in\mathcal{B}_{i}\right)\leq\sqrt{\frac{1-\overline{\alpha}_{t}}{\overline{\alpha}_{t}}}\exp\left(-\frac{C_{3}k\log T}{32}\right).\label{eq:proof-delta-bound}
\end{align}

\subsubsection{Step 1: deriving an upper bound for $\Delta(x_{0})$}

Suppose that $x_{0}\in\mathcal{B}_{i}$ for some $1\leq i\leq N_{\varepsilon}$
(notice that here we are not requiring that $i\in\mathcal{I}$). We
will bound each of $|\Delta_{i}(x_{0})$ for $i=1,2,3,4$. We first
record two basic facts about the step sizes, which are immediate consequences
of Lemma~\ref{lemma:step-size}:
\begin{align}
\frac{\sqrt{\overline{\alpha}_{t}\left(1-\alpha_{t}\right)}}{\alpha_{t}-\overline{\alpha}_{t}} & =\sqrt{\frac{\overline{\alpha}_{t}}{1-\overline{\alpha}_{t}}}\sqrt{1+\frac{1-\alpha_{t}}{\alpha_{t}-\overline{\alpha}_{t}}}\sqrt{\frac{1-\alpha_{t}}{\alpha_{t}-\overline{\alpha}_{t}}}\leq\sqrt{\frac{\overline{\alpha}_{t}}{1-\overline{\alpha}_{t}}}\sqrt{1+\frac{8c_{1}\log T}{T}}\sqrt{\frac{8c_{1}\log T}{T}}\nonumber \\
 & \leq3\sqrt{\frac{c_{1}\log T}{T}}\sqrt{\frac{\overline{\alpha}_{t}}{1-\overline{\alpha}_{t}}},\label{eq:proof-lemma-set-step-size-1}
\end{align}
as long as $T$ is sufficiently large, and
\begin{align}
\frac{\left(1-\alpha_{t}\right)\sqrt{\overline{\alpha}_{t}}}{\left(\alpha_{t}-\overline{\alpha}_{t}\right)\sqrt{1-\overline{\alpha}_{t}}} & \leq\frac{8c_{1}\log T}{T}\sqrt{\frac{\overline{\alpha}_{t}}{1-\overline{\alpha}_{t}}}.\label{eq:proof-lemma-set-step-size-2}
\end{align}
We learn from (\ref{eq:proof-lemma-omega-inner}) that
\begin{equation}
\max\left\{ \left|(\widehat{x}_{0}-x_{0})^{\top}\omega\right|,\left|(\widehat{x}_{0}-x_{0})^{\top}\omega'\right|\right\} \leq\sqrt{C_{1}k\log T}\Vert\widehat{x}_{0}-x_{0}\Vert_{2}+\big(4\sqrt{d}+4\sqrt{C_{1}k\log T}\big)\varepsilon.\label{eq:proof-lemma-set-11}
\end{equation}
We also have
\begin{align}
\Vert\widehat{x}_{0}-x_{0}\Vert_{2} & \leq\Vert\overline{x}_{0}-x_{0}\Vert_{2}+\Vert\delta\Vert_{2}\overset{\text{(i)}}{\leq}\Vert x_{i(x_{t})}^{\star}-x_{i}^{\star}\Vert_{2}+\Vert x_{i(x_{t})}^{\star}-\overline{x}_{0}\Vert_{2}+\varepsilon+\Vert\delta\Vert_{2}\nonumber \\
 & \overset{\text{(ii)}}{\leq}\Vert x_{i(x_{t})}^{\star}-x_{i}^{\star}\Vert_{2}+3\sqrt{\frac{C_{3}k(1-\overline{\alpha}_{t})\log T}{\overline{\alpha}_{t}}}+\varepsilon+\sqrt{\frac{1-\overline{\alpha}_{t}}{\overline{\alpha}_{t}}}\exp\left(-\frac{C_{3}k\log T}{32}\right)\nonumber \\
 & \overset{\text{(iii)}}{\leq}\Vert x_{i(x_{t})}^{\star}-x_{i}^{\star}\Vert_{2}+4\sqrt{\frac{C_{3}k(1-\overline{\alpha}_{t})\log T}{\overline{\alpha}_{t}}}.\label{eq:proof-lemma-set-12}
\end{align}
Here step (i) holds since $x_{0}\in\mathcal{B}_{i}$, hence $\Vert x_{0}-x_{i}^{\star}\Vert_{2}\leq\varepsilon$;
step (ii) follows from (\ref{eq:proof-lemma-set-Xt-dist}) and the
fact that $x_{i(x_{t})}^{\star},\overline{x}_{0}\in\mathcal{X}_{t}(x_{t})$;
while step (iii) follows from (\ref{eq:eps-condition}) and holds
provided that $C_{3}$ is sufficiently large. Then we have\begin{subequations}\label{eq:proof-lemma-set-Delta-all}
\begin{align}
\left|\Delta_{1}(x_{0})\right| & \leq\frac{\sqrt{\overline{\alpha}_{t}\left(1-\alpha_{t}\right)}}{\alpha_{t}-\overline{\alpha}_{t}}\left|(\widehat{x}_{0}-x_{0})^{\top}\omega\right|\nonumber \\
 & \overset{\text{(a)}}{\leq}3\sqrt{\frac{c_{1}\log T}{T}}\sqrt{\frac{\overline{\alpha}_{t}}{1-\overline{\alpha}_{t}}}\left(\sqrt{C_{1}k\log T}\Vert\widehat{x}_{0}-x_{0}\Vert_{2}+\big(4\sqrt{d}+4\sqrt{C_{1}k\log T}\big)\varepsilon\right)\nonumber \\
 & \overset{\text{(b)}}{\leq}4\sqrt{\frac{c_{1}\log T}{T}}\sqrt{\frac{\overline{\alpha}_{t}}{1-\overline{\alpha}_{t}}}\left(\sqrt{C_{1}k\log T}\Vert x_{i(x_{t})}^{\star}-x_{i}^{\star}\Vert_{2}+4\sqrt{C_{1}C_{3}}k\log T\sqrt{\frac{1-\overline{\alpha}_{t}}{\overline{\alpha}_{t}}}\right).\label{eq:proof-lemma-set-Delta-1}
\end{align}
Here step (a) follows from (\ref{eq:proof-lemma-set-step-size-1})
and (\ref{eq:proof-lemma-set-11}); while step (b) utilizes (\ref{eq:proof-lemma-set-12}),
(\ref{eq:eps-condition}). Similarly we can use (\ref{eq:proof-lemma-set-step-size-2})
to show that
\begin{align}
\left|\Delta_{2}(x_{0})\right| & \leq\frac{9c_{1}\log T}{T}\sqrt{\frac{\overline{\alpha}_{t}}{1-\overline{\alpha}_{t}}}\left(\sqrt{C_{1}k\log T}\Vert x_{i(x_{t})}^{\star}-x_{i}^{\star}\Vert_{2}+4\sqrt{C_{1}C_{3}}k\log T\sqrt{\frac{1-\overline{\alpha}_{t}}{\overline{\alpha}_{t}}}\right).\label{eq:proof-lemma-set-Delta-2}
\end{align}
Notice that
\begin{align*}
\left|(x_{0}(x_{t})-\widehat{x}_{0})^{\top}(\widehat{x}_{0}-x_{0})\right| & \overset{\text{(i)}}{\leq}\left\Vert x_{0}(x_{t})-\widehat{x}_{0}\right\Vert _{2}\left\Vert \widehat{x}_{0}-x_{0}\right\Vert _{2}\leq\left(\left\Vert x_{0}(x_{t})-\overline{x}_{0}\right\Vert _{2}+\left\Vert \delta\right\Vert _{2}\right)\left\Vert \widehat{x}_{0}-x_{0}\right\Vert _{2}\\
 & \overset{\text{(ii)}}{\leq}\Bigg[3\sqrt{\frac{C_{3}k(1-\overline{\alpha}_{t})\log T}{\overline{\alpha}_{t}}}+\sqrt{\frac{1-\overline{\alpha}_{t}}{\overline{\alpha}_{t}}}\exp\left(-\frac{C_{3}k\log T}{32}\right)\Bigg]\left\Vert \widehat{x}_{0}-x_{0}\right\Vert _{2}\\
 & \overset{\text{(iii)}}{\leq}4\sqrt{\frac{C_{3}k(1-\overline{\alpha}_{t})\log T}{\overline{\alpha}_{t}}}\left\Vert \widehat{x}_{0}-x_{0}\right\Vert _{2},
\end{align*}
where step (i) utilizes the Cauchy-Schwarz inequality; step (ii) follows
from (\ref{eq:proof-lemma-set-Xt-dist}), (\ref{eq:proof-delta-bound})
and the fact that $x_{0}(x_{t}),\overline{x}_{0}\in\mathcal{X}_{t}(x_{t})$;
step (iii) holds provided that $C_{3}$ is sufficiently large. Therefore
we have
\begin{align}
\left|\Delta_{3}(x_{0})\right| & \leq\frac{\left(1-\alpha_{t}\right)\overline{\alpha}_{t}}{\left(\alpha_{t}-\overline{\alpha}_{t}\right)\left(1-\overline{\alpha}_{t}\right)}\left|(x_{0}(x_{t})-\widehat{x}_{0})^{\top}(\widehat{x}_{0}-x_{0})\right|\nonumber \\
 & \leq\frac{\left(1-\alpha_{t}\right)\overline{\alpha}_{t}}{\left(\alpha_{t}-\overline{\alpha}_{t}\right)\left(1-\overline{\alpha}_{t}\right)}\cdot4\sqrt{\frac{C_{3}k(1-\overline{\alpha}_{t})\log T}{\overline{\alpha}_{t}}}\left\Vert \widehat{x}_{0}-x_{0}\right\Vert _{2}\nonumber \\
 & \leq32c_{1}\sqrt{C_{3}}\frac{\log T}{T}\sqrt{\frac{\overline{\alpha}_{t}}{1-\overline{\alpha}_{t}}}\sqrt{k\log T}\left(\Vert x_{i(x_{t})}^{\star}-x_{i}^{\star}\Vert_{2}+4\sqrt{\frac{C_{3}k(1-\overline{\alpha}_{t})\log T}{\overline{\alpha}_{t}}}\right),\label{eq:proof-lemma-set-Delta-3}
\end{align}
where the last relation follows from Lemma~\ref{lemma:step-size}
and (\ref{eq:proof-lemma-set-12}). Finally we have
\begin{align}
\left|\Delta_{4}(x_{0})\right| & \leq\frac{\left(1-\alpha_{t}\right)\overline{\alpha}_{t}}{2\left(\alpha_{t}-\overline{\alpha}_{t}\right)\left(1-\overline{\alpha}_{t}\right)}\left\Vert \widehat{x}_{0}-x_{0}\right\Vert _{2}^{2}\nonumber \\
 & \leq\frac{8c_{1}\log T}{T}\frac{\overline{\alpha}_{t}}{1-\overline{\alpha}_{t}}\left(\Vert x_{i(x_{t})}^{\star}-x_{i}^{\star}\Vert_{2}^{2}+16\frac{C_{3}k(1-\overline{\alpha}_{t})\log T}{\overline{\alpha}_{t}}\right).\label{eq:proof-lemma-set-Delta-4}
\end{align}
\end{subequations}Here step (a) follows from (\ref{eq:proof-lemma-set-12}),
step (b) follows from (\ref{eq:proof-lemma-set-Xt-dist}) and the
fact that $x_{i(x_{t})}^{\star},x_{i}^{\star}\in\mathcal{X}_{t}(x_{t})$Taking
the bounds in (\ref{eq:proof-lemma-set-Delta-all}) collectively leads
to 
\begin{align}
\left|\Delta\left(x_{0}\right)\right| & \leq5\sqrt{c_{1}C_{1}}\sqrt{\frac{k}{T}}\log T\left(\sqrt{\frac{\overline{\alpha}_{t}}{1-\overline{\alpha}_{t}}}\Vert x_{i(x_{t})}^{\star}-x_{i}^{\star}\Vert_{2}+4\sqrt{C_{3}k\log T}\right)\nonumber \\
 & \qquad+\frac{8c_{1}\log T}{T}\frac{\overline{\alpha}_{t}}{1-\overline{\alpha}_{t}}\Vert x_{i(x_{t})}^{\star}-x_{i}^{\star}\Vert_{2}^{2}.\label{eq:proof-lemma-set-13}
\end{align}
provided that $T$ is sufficiently large.

\subsubsection{Step 2: bounding $I_{1}$}

For $x_{0}\in\mathcal{X}_{t}(x_{t})$, we know that $x_{0}\in\mathcal{B}_{i}$
for some $i\in\mathcal{I}(x_{t};C_{3})$, hence $\overline{\alpha}_{t}\Vert x_{i}^{\star}-x_{i(x_{t})}^{\star}\Vert_{2}^{2}\leq C_{3}k(1-\overline{\alpha}_{t})\log T$.
This combined with (\ref{eq:proof-lemma-set-13}) gives
\begin{align}
\left|\Delta\left(x_{0}\right)\right| & \leq25\sqrt{c_{1}C_{1}C_{3}}\sqrt{\frac{k^{2}\log^{3}T}{T}}+\frac{8c_{1}C_{3}k\log^{2}T}{T}\leq26\sqrt{c_{1}C_{1}C_{3}}\sqrt{\frac{k^{2}\log^{3}T}{T}}\label{eq:proof-lemma-set-14}
\end{align}
provided that $T\gg k^{2}\log^{3}T$. Similarly we can check that
for each $1\leq i\leq4$, $|\Delta_{i}(x_{0})|\leq1$. Then we know
that for $x_{0}\in\mathcal{X}_{t}(x_{t})$, we have $\exp(\Delta(x_{0}))\leq1+\Delta(x_{0})+\Delta^{2}(x_{0})$
as long as $T\gg k^{2}\log^{3}T$. Hence
\begin{align*}
I_{1} & \leq1+\int_{\mathcal{X}_{t}(x_{t})}p_{X_{0}|X_{t}}\left(x_{0}\mymid x_{t}\right)\Delta\left(x_{0}\right)\mathrm{d}x_{0}+\int_{\mathcal{X}_{t}(x_{t})}p_{X_{0}|X_{t}}\left(x_{0}\mymid x_{t}\right)\Delta^{2}\left(x_{0}\right)\mathrm{d}x_{0}\\
 & =1+\int p_{X_{0}|X_{t}}\left(x_{0}\mymid x_{t}\right)\Delta_{1}\left(x_{0}\right)\mathrm{d}x_{0}-\int_{\mathcal{Y}_{t}(x_{t})}p_{X_{0}|X_{t}}\left(x_{0}\mymid x_{t}\right)\Delta_{1}\left(x_{0}\right)\mathrm{d}x_{0}\\
 & \qquad+\int_{\mathcal{X}_{t}(x_{t})}p_{X_{0}|X_{t}}\left(x_{0}\mymid x_{t}\right)\left[\Delta_{2}\left(x_{0}\right)+\Delta_{3}\left(x_{0}\right)+\Delta_{4}\left(x_{0}\right)\right]\mathrm{d}x_{0}+\int_{\mathcal{X}_{t}(x_{t})}p_{X_{0}|X_{t}}\left(x_{0}\mymid x_{t}\right)\Delta^{2}\left(x_{0}\right)\mathrm{d}x_{0}\\
 & \leq1+\underbrace{\int_{\mathcal{Y}_{t}(x_{t})}p_{X_{0}|X_{t}}\left(x_{0}\mymid x_{t}\right)\left|\Delta_{1}\left(x_{0}\right)\right|\mathrm{d}x_{0}}_{\eqqcolon I_{1,1}}\\
 & \qquad+\underbrace{\int_{\mathcal{X}_{t}(x_{t})}p_{X_{0}|X_{t}}\left(x_{0}\mymid x_{t}\right)\left[\left|\Delta_{2}\left(x_{0}\right)+\Delta_{3}\left(x_{0}\right)+\Delta_{4}\left(x_{0}\right)\right|+\Delta^{2}\left(x_{0}\right)\right]\mathrm{d}x_{0}}_{\eqqcolon I_{1,2}}.
\end{align*}
Here the last step follows from the fact that $\int p_{X_{0}|X_{t}}(x_{0}\mymid x_{t})\Delta_{1}(x_{0})\mathrm{d}x_{0}=0$.
The integral $I_{1,1}$ can be upper bounded similar to $I_{2}$,
hence we defer its analysis to the next section. For the integral
$I_{1,2}$, we have
\begin{align}
I_{1,2} & \leq\max_{x_{0}\in\mathcal{X}_{t}(x_{t})}\left\{ \left|\Delta_{2}\left(x_{0}\right)+\Delta_{3}\left(x_{0}\right)+\Delta_{4}\left(x_{0}\right)\right|+\Delta^{2}\left(x_{0}\right)\right\} \nonumber \\
 & \overset{\text{(i)}}{\leq}\max_{x_{0}\in\mathcal{X}_{t}(x_{t})}\left\{ 4\Delta_{1}^{2}\left(x_{0}\right)+5\left|\Delta_{2}\left(x_{0}\right)\right|+5\left|\Delta_{3}\left(x_{0}\right)\right|+5\left|\Delta_{4}\left(x_{0}\right)\right|\right\} \nonumber \\
 & \overset{\text{(ii)}}{\leq}128\frac{c_{1}\log T}{T}\frac{\overline{\alpha}_{t}}{1-\overline{\alpha}_{t}}\left(C_{1}k\log T\frac{C_{3}k\left(1-\overline{\alpha}_{t}\right)\log T}{\overline{\alpha}_{t}}+16C_{1}C_{3}k^{2}\log^{2}T\frac{1-\overline{\alpha}_{t}}{\overline{\alpha}_{t}}\right)\nonumber \\
 & \qquad+\frac{45c_{1}\log T}{T}\sqrt{\frac{\overline{\alpha}_{t}}{1-\overline{\alpha}_{t}}}\left(\sqrt{C_{1}k\log T}\sqrt{\frac{C_{3}k\left(1-\overline{\alpha}_{t}\right)\log T}{\overline{\alpha}_{t}}}+4\sqrt{C_{1}C_{3}}k\log T\sqrt{\frac{1-\overline{\alpha}_{t}}{\overline{\alpha}_{t}}}\right)\nonumber \\
 & \qquad+160c_{1}\sqrt{C_{3}}\frac{\log T}{T}\sqrt{\frac{\overline{\alpha}_{t}}{1-\overline{\alpha}_{t}}}\sqrt{k\log T}\left(\sqrt{\frac{C_{3}k\left(1-\overline{\alpha}_{t}\right)\log T}{\overline{\alpha}_{t}}}+4\sqrt{\frac{C_{3}k(1-\overline{\alpha}_{t})\log T}{\overline{\alpha}_{t}}}\right)\nonumber \\
 & \qquad+\frac{40c_{1}\log T}{T}\frac{\overline{\alpha}_{t}}{1-\overline{\alpha}_{t}}\left(\frac{C_{3}k\left(1-\overline{\alpha}_{t}\right)\log T}{\overline{\alpha}_{t}}+16\frac{C_{3}k(1-\overline{\alpha}_{t})\log T}{\overline{\alpha}_{t}}\right)\nonumber \\
 & \leq2181c_{1}C_{1}C_{3}\frac{k^{2}\log^{3}T}{T}.\label{eq:proof-lemma-set-15}
\end{align}
Here step (i) follows from the Cauchy-Schwarz inequality and the facts
that $|\Delta_{i}(x_{0})|\leq1$ for $i=2,3,4$, while step (ii) follows
from the bounds (\ref{eq:proof-lemma-set-Delta-all}) and the fact
that $\overline{\alpha}_{t}\Vert x_{i}^{\star}-x_{i(x_{t})}^{\star}\Vert_{2}^{2}\leq C_{3}k(1-\overline{\alpha}_{t})\log T$.

\subsubsection{Step 3: bounding $I_{2}$. }

For $x_{0}\in\mathcal{Y}_{t}(x_{t})$, we know that $x_{0}\in\mathcal{B}_{i}$
for some $i\notin\mathcal{I}(x_{t};C_{3})$, hence $\overline{\alpha}_{t}\Vert x_{i}^{\star}-x_{i(x_{t})}^{\star}\Vert_{2}^{2}>C_{3}k(1-\overline{\alpha}_{t})\log T$.
This combined with (\ref{eq:proof-lemma-set-13}) gives
\begin{align}
\left|\Delta\left(x_{0}\right)\right| & \leq\left(5\sqrt{\frac{c_{1}C_{1}}{C_{3}}}\sqrt{\frac{\log T}{T}}+20\sqrt{\frac{c_{1}C_{1}}{C_{3}}}\sqrt{\frac{\log T}{T}}+\frac{8c_{1}\log T}{T}\right)\frac{\overline{\alpha}_{t}}{1-\overline{\alpha}_{t}}\Vert x_{i(x_{t})}^{\star}-x_{i}^{\star}\Vert_{2}^{2}\nonumber \\
 & \leq25\sqrt{\frac{c_{1}C_{1}}{C_{3}}}\sqrt{\frac{\log T}{T}}\frac{\overline{\alpha}_{t}}{1-\overline{\alpha}_{t}}\Vert x_{i(x_{t})}^{\star}-x_{i}^{\star}\Vert_{2}^{2}\label{eq:proof-lemma-set-16}
\end{align}
as long as $T$ is sufficiently large. Therefore we have
\begin{align}
I_{2} & =\int_{\mathcal{Y}_{t}(x_{t})}p_{X_{0}|X_{t}}\left(x_{0}\mymid x_{t}\right)\exp\left(\Delta\left(x_{0}\right)\right)\mathrm{d}x_{0}\leq\sum_{i\notin\mathcal{I}(x_{t};C_{3})}\mathbb{P}\left(X_{0}\in\mathcal{B}_{i}\mymid X_{t}=x_{t}\right)\max_{x_{0}\in\mathcal{B}_{i}}\exp\left(\Delta\left(x_{0}\right)\right)\nonumber \\
 & \overset{\text{(i)}}{\leq}\sum_{i\notin\mathcal{I}(x_{t};C_{3})}\exp\left(-\frac{\overline{\alpha}_{t}}{16\left(1-\overline{\alpha}_{t}\right)}\Vert x_{i(x_{t})}^{\star}-x_{i}^{\star}\Vert_{2}^{2}+25\sqrt{\frac{c_{1}C_{1}}{C_{3}}}\sqrt{\frac{\log T}{T}}\frac{\overline{\alpha}_{t}}{1-\overline{\alpha}_{t}}\Vert x_{i(x_{t})}^{\star}-x_{i}^{\star}\Vert_{2}^{2}\right)\mathbb{P}\left(X_{0}\in\mathcal{B}_{i}\right)\nonumber \\
 & \overset{\text{(ii)}}{\leq}\sum_{i\notin\mathcal{I}(x_{t};C_{3})}\exp\left(-\frac{\overline{\alpha}_{t}}{32\left(1-\overline{\alpha}_{t}\right)}\Vert x_{i(x_{t})}^{\star}-x_{i}^{\star}\Vert_{2}^{2}\right)\mathbb{P}\left(X_{0}\in\mathcal{B}_{i}\right)\nonumber \\
 & \overset{\text{(iii)}}{\leq}\exp\left(-\frac{C_{3}}{32}k\log T\right).\label{eq:proof-lemma-set-17}
\end{align}
Here step (i) follows from (\ref{eq:proof-lemma-set-10'}) and (\ref{eq:proof-lemma-set-16});
step (ii) holds as long as $T$ is sufficiently large; while step
(iii) uses the fact that $\overline{\alpha}_{t}\Vert x_{i}^{\star}-x_{i(x_{t})}^{\star}\Vert_{2}^{2}>C_{3}k(1-\overline{\alpha}_{t})\log T$
for $i\notin\mathcal{I}(x_{t};C_{3})$. By similar analysis, we can
show that
\begin{align}
I_{1,1} & =\int_{\mathcal{Y}_{t}(x_{t})}p_{X_{0}|X_{t}}\left(x_{0}\mymid x_{t}\right)\left|\Delta_{1}\left(x_{0}\right)\right|\mathrm{d}x_{0}\leq\sum_{i\notin\mathcal{I}(x_{t};C_{3})}\mathbb{P}\left(X_{0}\in\mathcal{B}_{i}\mymid X_{t}=x_{t}\right)\max_{x_{0}\in\mathcal{B}_{i}}\left|\Delta_{1}\left(x_{0}\right)\right|\nonumber \\
 & \overset{\text{(a)}}{\leq}20\sqrt{\frac{c_{1}C_{1}k\log^{2}T}{T}}\sqrt{\frac{\overline{\alpha}_{t}}{1-\overline{\alpha}_{t}}}\sum_{i\notin\mathcal{I}(x_{t};C_{3})}\mathbb{P}\left(X_{0}\in\mathcal{B}_{i}\mymid X_{t}=x_{t}\right)\Vert x_{i(x_{t})}^{\star}-x_{i}^{\star}\Vert_{2}\nonumber \\
 & \overset{\text{(b)}}{\leq}20\sqrt{\frac{c_{1}C_{1}k\log^{2}T}{T}}\sqrt{\frac{\overline{\alpha}_{t}}{1-\overline{\alpha}_{t}}}\sum_{i\notin\mathcal{I}(x_{t};C_{3})}\exp\left(-\frac{\overline{\alpha}_{t}}{16\left(1-\overline{\alpha}_{t}\right)}\Vert x_{i(x_{t})}^{\star}-x_{i}^{\star}\Vert_{2}^{2}\right)\mathbb{P}\left(X_{0}\in\mathcal{B}_{i}\right)\Vert x_{i(x_{t})}^{\star}-x_{i}^{\star}\Vert_{2}\nonumber \\
 & \overset{\text{(c)}}{\leq}20\sqrt{\frac{c_{1}C_{1}k\log^{2}T}{T}}\sum_{i\notin\mathcal{I}(x_{t};C_{3})}\exp\left(-\frac{\overline{\alpha}_{t}}{32\left(1-\overline{\alpha}_{t}\right)}\Vert x_{i(x_{t})}^{\star}-x_{i}^{\star}\Vert_{2}^{2}\right)\mathbb{P}\left(X_{0}\in\mathcal{B}_{i}\right)\nonumber \\
 & \overset{\text{(d)}}{\leq}20\sqrt{\frac{c_{1}C_{1}k\log^{2}T}{T}}\exp\left(-\frac{C_{3}}{32}k\log T\right).\label{eq:proof-lemma-set-18}
\end{align}
Here step (a) follows from (\ref{eq:proof-lemma-set-Delta-1}) and
the fact that $\overline{\alpha}_{t}\Vert x_{i}^{\star}-x_{i(x_{t})}^{\star}\Vert_{2}^{2}>C_{3}k(1-\overline{\alpha}_{t})\log T$
for $i\notin\mathcal{I}(x_{t};C_{3})$; step (b) follows from (\ref{eq:proof-lemma-set-10'});
step (c) holds provided that $C_{3}$ is sufficiently large; step
(d) follows again from the fact that $\overline{\alpha}_{t}\Vert x_{i}^{\star}-x_{i(x_{t})}^{\star}\Vert_{2}^{2}>C_{3}k(1-\overline{\alpha}_{t})\log T$
for $i\notin\mathcal{I}(x_{t};C_{3})$.

\subsubsection{Step 4: putting everything together}

Taking (\ref{eq:proof-lemma-set-15}), (\ref{eq:proof-lemma-set-17})
and (\ref{eq:proof-lemma-set-18}) collectively, we have
\begin{align*}
\int_{x_{0}}p_{X_{0}|X_{t}}\left(x_{0}\mymid x_{t}\right)\exp\left(\Delta\left(x_{t},x_{t-1},x_{0}\right)\right)\mathrm{d}x_{0} & =I_{1}+I_{2}\leq1+I_{1,1}+I_{1,2}+I_{2}\\
 & \leq1+2182c_{1}C_{1}C_{3}\frac{k^{2}\log^{3}T}{T},
\end{align*}
provided that $T$ is sufficiently large. By similar argument, i.e.,
using the lower bounding $\exp(\Delta(x_{0}))\geq1+\Delta(x_{0})-\Delta^{2}(x_{0})$
in Step 2 and repeat the same analysis, we can show that 
\begin{align*}
\int_{x_{0}}p_{X_{0}|X_{t}}\left(x_{0}\mymid x_{t}\right)\exp\left(\Delta\left(x_{0}\right)\right)\mathrm{d}x_{0} & \geq1-2182c_{1}C_{1}C_{3}\frac{k^{2}\log^{3}T}{T}.
\end{align*}
This gives the desired result.

\subsection{Proof of Lemma \ref{lemma:sde-coarse} \label{subsec:proof-lemma-sde-coarse}}

Recall that
\begin{align*}
\log\frac{p_{X_{t-1}|X_{t}}\left(x_{t-1}\mymid x_{t}\right)}{p_{Y_{t-1}^{\star}|Y_{t}}\left(x_{t-1}\mymid x_{t}\right)} & =\log\left[\int_{x_{0}}p_{X_{0}|X_{t}}\left(x_{0}\mymid x_{t}\right)\exp\left(\Delta\left(x_{t},x_{t-1},x_{0}\right)\right)\mathrm{d}x_{0}\right].
\end{align*}
For any $x_{0}\in\mathcal{X}$, by the definition of $\Delta\left(x_{t},x_{t-1},x_{0}\right)$
in (\ref{eq:Delta-defn}), we have
\begin{align*}
\left|\Delta\left(x_{t},x_{t-1},x_{0}\right)\right| & \leq\frac{\sqrt{\overline{\alpha}_{t}}}{\alpha_{t}-\overline{\alpha}_{t}}\Vert\sqrt{\alpha_{t}}x_{t-1}-x_{t}\Vert_{2}\Vert\widehat{x}_{0}-x_{0}\Vert_{2}+\frac{\left(1-\alpha_{t}\right)\sqrt{\overline{\alpha}_{t}}}{\left(\alpha_{t}-\overline{\alpha}_{t}\right)\left(1-\overline{\alpha}_{t}\right)}\Vert x_{t}-\sqrt{\overline{\alpha}_{t}}\widehat{x}_{0}\Vert_{2}\Vert\widehat{x}_{0}-x_{0}\Vert_{2}\\
 & \qquad+\frac{\left(1-\alpha_{t}\right)\overline{\alpha}_{t}}{2\left(\alpha_{t}-\overline{\alpha}_{t}\right)\left(1-\overline{\alpha}_{t}\right)}\Vert\widehat{x}_{0}-x_{0}\Vert_{2}^{2}\\
 & \overset{\text{(i)}}{\leq}2R\frac{\sqrt{\overline{\alpha}_{t}}}{\alpha_{t}-\overline{\alpha}_{t}}\Vert\sqrt{\alpha_{t}}x_{t-1}-x_{t}\Vert_{2}+2R\frac{\left(1-\alpha_{t}\right)\sqrt{\overline{\alpha}_{t}}}{\left(\alpha_{t}-\overline{\alpha}_{t}\right)\left(1-\overline{\alpha}_{t}\right)}\Vert x_{t}\Vert_{2}+\frac{\left(1-\alpha_{t}\right)\overline{\alpha}_{t}}{\left(\alpha_{t}-\overline{\alpha}_{t}\right)\left(1-\overline{\alpha}_{t}\right)}4R^{2}\\
 & \overset{\text{(ii)}}{\leq}4RT^{c_{0}}\Vert\sqrt{\alpha_{t}}x_{t-1}-x_{t}\Vert_{2}+16c_{1}RT^{c_{0}-1}\log T\Vert x_{t}\Vert_{2}+32c_{1}R^{2}T^{c_{0}-1}\log T.
\end{align*}
Here step (i) follows from $\widehat{x}_{0},x_{0}\in\mathcal{X}$,
hence $\max\{\Vert\widehat{x}_{0}\Vert_{2},\Vert x_{0}\Vert_{2}\}\leq R$;
while step (ii) follows from the facts that, for $2\leq t\leq T$,
\[
\frac{\sqrt{\overline{\alpha}_{t}}}{\alpha_{t}-\overline{\alpha}_{t}}\leq\frac{1}{\alpha_{t}-\prod_{i=1}^{t}\alpha_{i}}=\frac{1}{\alpha_{t}\left(1-\prod_{i=1}^{t-1}\alpha_{i}\right)}\leq\frac{2}{1-\alpha_{1}}\leq2T^{c_{0}},
\]
and in view of Lemma~\ref{lemma:step-size}, 
\[
\frac{\left(1-\alpha_{t}\right)\overline{\alpha}_{t}}{\left(\alpha_{t}-\overline{\alpha}_{t}\right)\left(1-\overline{\alpha}_{t}\right)}\leq\frac{\left(1-\alpha_{t}\right)\sqrt{\overline{\alpha}_{t}}}{\left(\alpha_{t}-\overline{\alpha}_{t}\right)\left(1-\overline{\alpha}_{t}\right)}\leq\frac{8c_{1}\log T}{T}\frac{1}{1-\overline{\alpha}_{t}}\leq8c_{1}T^{c_{0}-1}\log T.
\]
Hence we have
\begin{align*}
\left|\log\frac{p_{X_{t-1}|X_{t}}\left(x_{t-1}\mymid x_{t}\right)}{p_{Y_{t-1}^{\star}|Y_{t}}\left(x_{t-1}\mymid x_{t}\right)}\right| & =\left|\log\left[\int_{x_{0}}p_{X_{0}|X_{t}}\left(x_{0}\mymid x_{t}\right)\exp\left(\Delta\left(x_{t},x_{t-1},x_{0}\right)\right)\mathrm{d}x_{0}\right]\right|\leq\sup_{x_{0}\in\mathcal{X}}\left|\Delta\left(x_{t},x_{t-1},x_{0}\right)\right|\\
 & \leq4RT^{c_{0}}\Vert\sqrt{\alpha_{t}}x_{t-1}-x_{t}\Vert_{2}+16c_{1}RT^{c_{0}-1}\log T\Vert x_{t}\Vert_{2}+32c_{1}R^{2}T^{c_{0}-1}\log T\\
 & \leq T^{c_{0}+2c_{R}}\left(\Vert\sqrt{\alpha_{t}}x_{t-1}-x_{t}\Vert_{2}+\Vert x_{t}\Vert_{2}+1\right)
\end{align*}
as long as $T$ is sufficiently large.

\subsection{Proof of Lemma \ref{lemma:SDE-Delta-t-1}\label{subsec:proof-lemma-sde-delta-t-1}}

Regarding $\Delta_{t,1}$, we first utilize Lemma~\ref{lemma:Delta-SDE}
to show that for any $(x_{t},x_{t-1})\in\mathcal{A}_{t}$, 
\[
\left|1-\frac{p_{Y_{t-1}^{\star}|Y_{t}}\left(x_{t-1}\mymid x_{t}\right)}{p_{X_{t-1}|X_{t}}\left(x_{t-1}\mymid x_{t}\right)}\right|\leq C_{5}\frac{k^{2}\log^{3}T}{T}.
\]
Since $\log(1-x)\geq-x-x^{2}$ holds for any $x\in[-1/2,1/2]$, we
know that when $T\gg k^{2}\log^{3}T$, we have
\begin{align*}
 & p_{X_{t-1}|X_{t}}\left(x_{t-1}\mymid x_{t}\right)\log\frac{p_{X_{t-1}|X_{t}}\left(x_{t-1}\mymid x_{t}\right)}{p_{Y_{t-1}^{\star}|Y_{t}}\left(x_{t-1}\mymid x_{t}\right)}=-p_{X_{t-1}|X_{t}}\left(x_{t-1}\mymid x_{t}\right)\log\left[1-\left(1-\frac{p_{Y_{t-1}^{\star}|Y_{t}}\left(x_{t-1}\mymid x_{t}\right)}{p_{X_{t-1}|X_{t}}\left(x_{t-1}\mymid x_{t}\right)}\right)\right]\\
 & \qquad\leq p_{X_{t-1}|X_{t}}\left(x_{t-1}\mymid x_{t}\right)\left[1-\frac{p_{Y_{t-1}^{\star}|Y_{t}}\left(x_{t-1}\mymid x_{t}\right)}{p_{X_{t-1}|X_{t}}\left(x_{t-1}\mymid x_{t}\right)}+\left(1-\frac{p_{Y_{t-1}^{\star}|Y_{t}}\left(x_{t-1}\mymid x_{t}\right)}{p_{X_{t-1}|X_{t}}\left(x_{t-1}\mymid x_{t}\right)}\right)^{2}\right]\\
 & \qquad=p_{X_{t-1}|X_{t}}\left(x_{t-1}\mymid x_{t}\right)-p_{Y_{t-1}^{\star}|Y_{t}}\left(x_{t-1}\mymid x_{t}\right)+p_{X_{t-1}|X_{t}}\left(x_{t-1}\mymid x_{t}\right)C_{5}^{2}\frac{k^{4}\log^{6}T}{T^{2}}
\end{align*}
Hence we have
\begin{align*}
\Delta_{t,1} & \leq\int_{\left(x_{t},x_{t-1}\right)\in\mathcal{A}_{t}}\left[-p_{Y_{t-1}^{\star}|Y_{t}}\left(x_{t-1}\mymid x_{t}\right)+p_{X_{t-1}|X_{t}}\left(x_{t-1}\mymid x_{t}\right)\right]p_{X_{t}}\left(x_{t}\right)\mathrm{d}x_{t-1}\mathrm{d}x_{t}+C_{5}^{2}\frac{k^{4}\log^{6}T}{T^{2}}\\
 & =\int_{\left(x_{t},x_{t-1}\right)\in\mathcal{A}_{t}^{\mathrm{c}}}\left[p_{Y_{t-1}^{\star}|Y_{t}}\left(x_{t-1}\mymid x_{t}\right)-p_{X_{t-1}|X_{t}}\left(x_{t-1}\mymid x_{t}\right)\right]p_{X_{t}}\left(x_{t}\right)\mathrm{d}x_{t-1}\mathrm{d}x_{t}+C_{5}^{2}\frac{k^{4}\log^{6}T}{T^{2}}\\
 & \leq\underbrace{\int_{\left(x_{t},x_{t-1}\right)\in\mathcal{A}_{t}^{\mathrm{c}}}p_{Y_{t-1}^{\star}|Y_{t}}\left(x_{t-1}\mymid x_{t}\right)p_{X_{t}}\left(x_{t}\right)\mathrm{d}x_{t-1}\mathrm{d}x_{t}}_{\eqqcolon\Delta_{t,3}}+C_{5}^{2}\frac{k^{4}\log^{6}T}{T^{2}}.
\end{align*}
Here the penultimate step follows from the fact that
\[
\int p_{Y_{t-1}^{\star}|Y_{t}}\left(x_{t-1}\mymid x_{t}\right)p_{X_{t}}\left(x_{t}\right)\mathrm{d}x_{t-1}\mathrm{d}x_{t}=\int p_{X_{t-1}|X_{t}}\left(x_{t-1}\mymid x_{t}\right)p_{X_{t}}\left(x_{t}\right)\mathrm{d}x_{t-1}\mathrm{d}x_{t}=1.
\]
It boils down to bounding $\Delta_{t,3}$. In view of (\ref{eq:proof-sde-7}),
we know that 
\begin{align*}
\Delta_{t,3} & =\int p_{Y_{t-1}^{\star}|Y_{t}}\left(x_{t-1}\mymid x_{t}\right)p_{X_{t}}\left(x_{t}\right)\ind\left\{ x_{t}\notin\mathcal{T}_{t}\right\} \mathrm{d}x_{t-1}\mathrm{d}x_{t}\\
 & \qquad+\int p_{Y_{t-1}^{\star}|Y_{t}}\left(x_{t-1}\mymid x_{t}\right)p_{X_{t}}\left(x_{t}\right)\ind\left\{ x_{t}\in\mathcal{T}_{t},\frac{x_{t}-\sqrt{\alpha_{t}}x_{t-1}}{\sqrt{1-\alpha_{t}}}\notin\mathcal{G}\right\} \mathrm{d}x_{t-1}\mathrm{d}x_{t}\\
 & =\mathbb{P}\left(X_{t}\notin\mathcal{T}_{t}\right)+\mathbb{P}\left(X_{t}\in\mathcal{T}_{t},\frac{X_{t}-\left(X_{t}+\eta_{t}^{\star}s_{t}^{\star}(X_{t})+\sigma_{t}^{\star}Z\right)}{\sqrt{1-\alpha_{t}}}\notin\mathcal{G}\right)\quad\text{where}\quad Z\sim\mathcal{N}(0,I_{d})\\
 & =\mathbb{P}\left(X_{t}\notin\mathcal{T}_{t}\right)+\mathbb{P}\left(X_{t}\in\mathcal{T}_{t},-\frac{\eta_{t}^{\star}s_{t}^{\star}(X_{t})+\sigma_{t}^{\star}Z}{\sqrt{1-\alpha_{t}}}\notin\mathcal{G}\right)
\end{align*}
Here we use the fact that $Y_{t-1}^{\star}\mymid Y_{t}=x_{t}\sim\mathcal{N}\big((x_{t}+\eta_{t}^{\star}s_{t}^{\star}(x_{t}))/\sqrt{\alpha_{t}},(\sigma_{t}^{\star2}/\alpha_{t})I_{d}\big)$.
Notice that
\[
-\frac{\eta_{t}^{\star}s_{t}^{\star}(X_{t})+\sigma_{t}^{\star}Z}{\sqrt{1-\alpha_{t}}}=-\sqrt{1-\alpha_{t}}s_{t}^{\star}\left(X_{t}\right)-\sqrt{\frac{\alpha_{t}-\overline{\alpha}_{t}}{1-\overline{\alpha}_{t}}}Z.
\]
The following claim is cricial for understanding this random variable.

\begin{claim} \label{claim:score}For any $x_{t}\in\mathcal{T}_{t}$,
we have
\[
\left\Vert \sqrt{1-\alpha_{t}}s_{t}^{\star}\left(x_{t}\right)\right\Vert _{2}\leq\frac{1}{2}\big(\sqrt{d}+\sqrt{C_{1}k\log T}\big),
\]
and for any $1\leq i\leq j\leq N_{\varepsilon}$, 
\[
\sqrt{1-\alpha_{t}}\vert(x_{i}^{\star}-x_{j}^{\star})^{\top}s_{t}^{\star}\left(x_{t}\right)\vert\leq\frac{1}{2}\sqrt{C_{1}k\log T}\Vert x_{i}^{\star}-x_{j}^{\star}\Vert_{2}.
\]
\end{claim}

\begin{proof}See Appendix~\ref{subsec:proof-claim-score}.\end{proof}

Since $Z\sim\mathcal{N}(0,I_{d})$, in view of Lemma~\ref{lemma:concentration},
with probability exceeding $1-\exp\left(-(C_{1}/64)k\log T\right)$,
\[
\sqrt{\frac{\alpha_{t}-\overline{\alpha}_{t}}{1-\overline{\alpha}_{t}}}\left\Vert Z\right\Vert _{2}\leq\left\Vert Z\right\Vert _{2}\leq\sqrt{d}+\frac{1}{2}\sqrt{C_{1}k\log T}
\]
and for any $1\leq i\leq j\leq N_{\varepsilon}$, 
\[
\sqrt{\frac{\alpha_{t}-\overline{\alpha}_{t}}{1-\overline{\alpha}_{t}}}\vert(x_{i}^{\star}-x_{j}^{\star})^{\top}Z\vert\leq\vert(x_{i}^{\star}-x_{j}^{\star})^{\top}Z\vert\leq\frac{1}{2}\sqrt{C_{1}k\log T}\Vert x_{i}^{\star}-x_{j}^{\star}\Vert_{2}.
\]
These combined with Claim~\ref{claim:score} allow us to show that
\[
\mathbb{P}\left(X_{t}\in\mathcal{T}_{t},-\frac{\eta_{t}^{\star}s_{t}^{\star}(X_{t})+\sigma_{t}^{\star}Z}{\sqrt{1-\alpha_{t}}}\notin\mathcal{G}\right)\leq\exp\left(-\frac{C_{1}}{64}k\log T\right).
\]
Taking the above inequality collectively with Lemma~\ref{lemma:At-SDE}
gives
\[
\Delta_{t,3}\leq\exp\left(-\frac{C_{1}}{4}k\log T\right)+\exp\left(-\frac{C_{1}}{64}k\log T\right)\leq2\exp\left(-\frac{C_{1}}{64}k\log T\right).
\]
Hence we have
\[
\Delta_{t,1}\leq\Delta_{t,3}+C_{5}^{2}\frac{k^{4}\log^{6}T}{T^{2}}\leq2\exp\left(-\frac{C_{1}}{64}k\log T\right)+C_{5}^{2}\frac{k^{4}\log^{6}T}{T^{2}}\leq2C_{5}^{2}\frac{k^{4}\log^{6}T}{T^{2}}
\]
as long as $T$ is sufficiently large.

\subsubsection{Proof of Claim \ref{claim:score} \label{subsec:proof-claim-score}}

Consider the decomposition $x_{t}=\sqrt{\overline{\alpha}_{t}}x_{0}(x_{t})+\sqrt{1-\overline{\alpha}_{t}}\omega$
as in Appendix~\ref{subsec:auxiliary-prelim}, where $x_{0}(x_{t})\in\mathcal{B}_{i(x_{t})}$
for some $i(x_{t})\in\mathcal{I}$ and $\omega\in\mathcal{G}$. Notice
that 
\begin{align*}
s_{t}^{\star}\left(x_{t}\right) & =-\frac{1}{1-\overline{\alpha}_{t}}\left(x_{t}-\sqrt{\overline{\alpha}_{t}}\widehat{x}_{0}\right)=-\frac{1}{1-\overline{\alpha}_{t}}\left[\sqrt{\overline{\alpha}_{t}}x_{0}(x_{t})+\sqrt{1-\overline{\alpha}_{t}}\omega-\sqrt{\overline{\alpha}_{t}}(\overline{x}_{0}+\delta)\right]\\
 & =-\frac{\sqrt{\overline{\alpha}_{t}}}{1-\overline{\alpha}_{t}}(x_{0}(x_{t})-\overline{x}_{0})-\frac{1}{\sqrt{1-\overline{\alpha}_{t}}}\omega+\frac{\sqrt{\overline{\alpha}_{t}}}{1-\overline{\alpha}_{t}}\delta.
\end{align*}
where $\widehat{x}_{0}\coloneqq\mathbb{E}[X_{0}\mymid X_{t}=x_{t}]$
is defined in (\ref{eq:x0-hat-defn}), whereas $\overline{x}_{0}\in\mathcal{X}_{t}(x_{t})$
and $\delta$ are defined in (\ref{eq:x0-hat-decom}). Therefore we
can check that
\begin{align*}
 & \left\Vert \sqrt{1-\alpha_{t}}s_{t}^{\star}\left(x_{t}\right)\right\Vert _{2}\leq\frac{\sqrt{\overline{\alpha}_{t}(1-\alpha_{t})}}{1-\overline{\alpha}_{t}}\Vert x_{0}(x_{t})-\overline{x}_{0}\Vert_{2}+\sqrt{\frac{1-\alpha_{t}}{1-\overline{\alpha}_{t}}}\Vert\omega\Vert_{2}+\frac{\sqrt{\overline{\alpha}_{t}(1-\alpha_{t})}}{1-\overline{\alpha}_{t}}\Vert\delta\Vert_{2}\\
 & \qquad\overset{\text{(i)}}{\leq}\frac{\sqrt{\overline{\alpha}_{t}(1-\alpha_{t})}}{1-\overline{\alpha}_{t}}3\sqrt{\frac{C_{3}k(1-\overline{\alpha}_{t})\log T}{\overline{\alpha}_{t}}}+\sqrt{\frac{1-\alpha_{t}}{1-\overline{\alpha}_{t}}}\big(\sqrt{d}+\sqrt{C_{1}k\log T}\big)+\sqrt{\frac{1-\alpha_{t}}{1-\overline{\alpha}_{t}}}\exp\left(-\frac{C_{3}k\log T}{32}\right)\\
 & \qquad\overset{\text{(ii)}}{\leq}\sqrt{\frac{8c_{1}\log T}{T}}\left[3\sqrt{C_{3}k\log T}+\sqrt{d}+\sqrt{C_{1}k\log T}+\exp\left(-\frac{C_{3}k\log T}{32}\right)\right]\\
 & \qquad\overset{\text{(iii)}}{\leq}\frac{1}{2}\left(\sqrt{d}+\sqrt{C_{1}k\log T}\right).
\end{align*}
Here step (i) follows from (\ref{eq:proof-lemma-set-Xt-dist}), the
fact that $\omega\in\mathcal{G}$, and (\ref{eq:proof-delta-bound});
step (ii) follows from Lemma~\ref{lemma:step-size}; while step (iii)
holds provided that $T$ is sufficiently large. In addition, for any
$1\leq i\leq j\leq N_{\varepsilon}$ we have
\begin{align*}
\sqrt{1-\alpha_{t}}\vert(x_{i}^{\star}-x_{j}^{\star})^{\top}s_{t}^{\star}\left(x_{t}\right)\vert & \overset{\text{(a)}}{\leq}\frac{\sqrt{\overline{\alpha}_{t}(1-\alpha_{t})}}{1-\overline{\alpha}_{t}}\Vert x_{0}(x_{t})-\overline{x}_{0}\Vert_{2}\Vert x_{i}^{\star}-x_{j}^{\star}\Vert_{2}+\sqrt{\frac{1-\alpha_{t}}{1-\overline{\alpha}_{t}}}\left|\omega^{\top}(x_{i}^{\star}-x_{j}^{\star})^{\top}\right|\\
 & \qquad+\frac{\sqrt{\overline{\alpha}_{t}(1-\alpha_{t})}}{1-\overline{\alpha}_{t}}\Vert\delta\Vert_{2}\Vert x_{i}^{\star}-x_{j}^{\star}\Vert_{2}\\
 & \overset{\text{(b)}}{\leq}\frac{\sqrt{\overline{\alpha}_{t}(1-\alpha_{t})}}{1-\overline{\alpha}_{t}}3\sqrt{\frac{C_{3}k(1-\overline{\alpha}_{t})\log T}{\overline{\alpha}_{t}}}\Vert x_{i}^{\star}-x_{j}^{\star}\Vert_{2}+\sqrt{\frac{1-\alpha_{t}}{1-\overline{\alpha}_{t}}}\sqrt{C_{1}k\log T}\Vert x_{i}^{\star}-x_{j}^{\star}\Vert_{2}\\
 & \qquad+\frac{\sqrt{\overline{\alpha}_{t}(1-\alpha_{t})}}{1-\overline{\alpha}_{t}}\sqrt{\frac{1-\overline{\alpha}_{t}}{\overline{\alpha}_{t}}}\exp\left(-\frac{C_{3}k\log T}{32}\right)\Vert x_{i}^{\star}-x_{j}^{\star}\Vert_{2}\\
 & \overset{\text{(c)}}{\leq}\sqrt{\frac{8c_{1}\log T}{T}}\left[3\sqrt{C_{3}k\log T}+\sqrt{C_{1}k\log T}+\exp\left(-\frac{C_{3}k\log T}{32}\right)\right]\Vert x_{i}^{\star}-x_{j}^{\star}\Vert_{2}\\
 & \overset{\text{(d)}}{\leq}\frac{1}{2}\sqrt{C_{1}k\log T}\Vert x_{i}^{\star}-x_{j}^{\star}\Vert_{2}.
\end{align*}
Here step (a) utilizes the Cauchy-Schwarz inequality; step (b) follows
from (\ref{eq:proof-lemma-set-Xt-dist}), the fact that $\omega\in\mathcal{G}$,
and (\ref{eq:proof-delta-bound}); step (c) follows from Lemma~\ref{lemma:step-size};
while step (d) holds when $T$ is sufficiently large. 

\subsection{Proof of Lemma \ref{lemma:SDE-Delta-t-2}\label{subsec:proof-lemma-sde-delta-t-2}}

We can upper bound $|\Delta_{t,2}|$ by
\begin{align}
\left|\Delta_{t,2}\right| & \overset{\text{(i)}}{\leq}T^{c_{0}+2c_{R}}\int\left(\Vert\sqrt{\alpha_{t}}x_{t-1}-x_{t}\Vert_{2}+\Vert x_{t}\Vert_{2}+1\right)p_{X_{t-1},X_{t}}\left(x_{t-1},x_{t}\right)\ind\left\{ (x_{t},x_{t-1})\notin\mathcal{A}_{t}\right\} \mathrm{d}x_{t-1}\mathrm{d}x_{t}\nonumber \\
 & =T^{c_{0}+2c_{R}}\mathbb{E}\left[\left(\Vert\sqrt{\alpha_{t}}X_{t-1}-X_{t}\Vert_{2}+\Vert X_{t}\Vert_{2}+1\right)\ind\left\{ (X_{t},X_{t-1})\notin\mathcal{A}_{t}\right\} \right]\nonumber \\
 & \overset{\text{(ii)}}{=}T^{c_{0}+2c_{R}}\mathbb{E}\left[\left(\sqrt{1-\alpha_{t}}\Vert W_{t}\Vert_{2}+\Vert X_{t}\Vert_{2}+1\right)\ind\left\{ (X_{t},X_{t-1})\notin\mathcal{A}_{t}\right\} \right]\nonumber \\
 & \overset{\text{(iii)}}{\leq}T^{c_{0}+2c_{R}}\sqrt{1-\alpha_{t}}\mathbb{E}^{1/2}\left[\Vert W_{t}\Vert_{2}^{2}\right]\mathbb{P}^{1/2}\left((X_{t},X_{t-1})\notin\mathcal{A}_{t}\right)+T^{c_{0}+2c_{R}}\mathbb{E}^{1/2}\left[\Vert X_{t}\Vert_{2}^{2}\right]\mathbb{P}^{1/2}\left((X_{t},X_{t-1})\notin\mathcal{A}_{t}\right)\nonumber \\
 & \qquad+T^{c_{0}+2c_{R}}\mathbb{P}\left((X_{t},X_{t-1})\notin\mathcal{A}_{t}\right).\label{eq:proof-Delta-t2-1}
\end{align}
Here step (i) follows from Lemma~\ref{lemma:sde-coarse}; step (ii)
follows from the update rule (\ref{eq:forward-update}); step (iii)
utilizes the Cauchy-Schwarz inequality. In view of (\ref{eq:forward-formula}),
we have
\begin{align}
\mathbb{E}\left[\Vert X_{t}\Vert_{2}^{2}\right] & =\mathbb{E}\left[\Vert\sqrt{\overline{\alpha}_{t}}X_{0}+\sqrt{1-\overline{\alpha}_{t}}\,\overline{W}_{t}\Vert_{2}^{2}\right]\leq\mathbb{E}\left[2\overline{\alpha}_{t}R^{2}+2\left(1-\overline{\alpha}_{t}\right)\Vert\overline{W}_{t}\Vert_{2}^{2}\right]\nonumber \\
 & =2\overline{\alpha}_{t}R^{2}+2\left(1-\overline{\alpha}_{t}\right)d\leq2R^{2}+2d,\label{eq:Xt-norm}
\end{align}
where we use the fact that $\mathbb{E}[\Vert\overline{W}_{t}\Vert_{2}^{2}]=d$.
Then we have
\begin{align*}
\left|\Delta_{t,2}\right| & \overset{\text{(a)}}{\leq}T^{c_{0}+2c_{R}}\left(\sqrt{d\left(1-\alpha_{t}\right)}+\sqrt{2R^{2}+2d}\right)\mathbb{P}^{1/2}\left((X_{t},X_{t-1})\notin\mathcal{A}_{t}\right)+T^{c_{0}+2c_{R}}\mathbb{P}\left((X_{t},X_{t-1})\notin\mathcal{A}_{t}\right)\\
 & \overset{\text{(b)}}{\leq}T^{c_{0}+2c_{R}}\left(2R+3\sqrt{d}\right)\exp\left(-\frac{C_{1}}{8}k\log T\right)+T^{c_{0}+2c_{R}}\exp\left(-\frac{C_{1}}{4}k\log T\right)\\
 & \overset{\text{(c)}}{\leq}\exp\left(-\frac{C_{1}}{16}k\log T\right)
\end{align*}
Here step (a) utilizes (\ref{eq:Xt-norm}) and the fact that $\mathbb{E}[\Vert W_{t}\Vert_{2}^{2}]=d$;
step (b) follows from Lemma~\ref{lemma:At-SDE}; while step (c) makes
use of the assumption that $k\geq\log d$ and holds provided that
$C_{1}\gg c_{0}+c_{R}$.

\subsection{Proof of Lemma \ref{lemma:sde-K}\label{subsec:proof-lemma-sde-K}}

We first decompose $K_{t}$ into
\begin{align*}
K_{t} & =\int p_{X_{t-1}|X_{t}}\left(x_{t-1}\mymid x_{t}\right)p_{X_{t}}\left(x_{t}\right)\left(x_{t-1}-\mu_{t}^{\star}\left(x_{t}\right)\right)^{\top}\varepsilon_{t}\left(x_{t}\right)\mathrm{d}x_{t-1}\mathrm{d}x_{t}\\
 & \overset{\text{(i)}}{=}\int\left(p_{X_{t-1}|X_{t}}\left(x_{t-1}\mymid x_{t}\right)-p_{Y_{t-1}^{\star}|Y_{t}}\left(x_{t-1}\mymid x_{t}\right)\right)p_{X_{t}}\left(x_{t}\right)\left(x_{t-1}-\mu_{t}^{\star}\left(x_{t}\right)\right)^{\top}\varepsilon_{t}\left(x_{t}\right)\mathrm{d}x_{t-1}\mathrm{d}x_{t}\\
 & \overset{\text{(ii)}}{=}\bigg(\int_{\mathcal{A}_{t}}+\int_{\mathcal{A}_{t}^{\mathrm{c}}}\bigg)\left(p_{X_{t-1}|X_{t}}\left(x_{t-1}\mymid x_{t}\right)-p_{Y_{t-1}^{\star}|Y_{t}}\left(x_{t-1}\mymid x_{t}\right)\right)p_{X_{t}}\left(x_{t}\right)\left(x_{t-1}-\mu_{t}^{\star}\left(x_{t}\right)\right)^{\top}\varepsilon_{t}\left(x_{t}\right)\mathrm{d}x_{t-1}\mathrm{d}x_{t}\\
 & \eqqcolon K_{t,1}+K_{t,2}.
\end{align*}
Here step (i) follows from the fact that $\int p_{Y_{t-1}^{\star}|Y_{t}}\left(x_{t-1}\mymid x_{t}\right)(x_{t-1}-\mu_{t}^{\star}(x_{t}))\mathrm{d}x_{t-1}=0$
for any $x_{t}\in\mathbb{R}^{d}$, and $K_{1}$ and $K_{2}$ are defined
to be the two integrals over $\mathcal{A}_{t}$ and $\mathcal{A}_{t}^{\mathrm{c}}$
in step (ii). The following two claims provide bounds for the two
integrals $K_{t,1}$ and $K_{t,2}$ respectively. 

\begin{claim}\label{claim:sde-K-1} Suppose that $T\gg k^{2}\log^{3}T$.
Then for each $2\leq t\leq T$, we have
\[
\left|K_{t,1}\right|\leq3C_{5}\frac{k^{2}\log^{3}T}{T}\sqrt{\frac{c_{1}\log T}{T}}\mathbb{E}_{x_{t}\sim q_{t}}\left[\Vert\varepsilon_{t}\left(x_{t}\right)\Vert_{2}\right].
\]
\end{claim}\begin{proof}See Appendix~\ref{subsec:proof-lemma-sde-K-1}.\end{proof}

\begin{claim}\label{claim:sde-K-2} Suppose that $T\gg1$. Then for
each $2\leq t\leq T$, we have
\[
\left|K_{t,2}\right|\leq2\exp\left(-\frac{C_{1}}{32}k\log T\right)\mathbb{E}_{x_{t}\sim q_{t}}^{1/2}\left[\Vert\varepsilon_{t}\left(x_{t}\right)\Vert_{2}^{2}\right].
\]
\end{claim}\begin{proof}See Appendix~\ref{subsec:proof-lemma-sde-K-2}.\end{proof}

Then we conclude that
\begin{align*}
\left|K_{t}\right| & \leq\left|K_{t,1}\right|+\left|K_{t,2}\right|\\
 & \overset{\text{(a)}}{\leq}3C_{5}\frac{k^{2}\log^{3}T}{T}\sqrt{\frac{c_{1}\log T}{T}}\mathbb{E}_{x_{t}\sim q_{t}}\left[\Vert\varepsilon_{t}\left(x_{t}\right)\Vert_{2}\right]+2\exp\left(-\frac{C_{1}}{32}k\log T\right)\mathbb{E}_{x_{t}\sim q_{t}}^{1/2}\left[\Vert\varepsilon_{t}\left(x_{t}\right)\Vert_{2}^{2}\right]\\
 & \overset{\text{(b)}}{\leq}4C_{5}\frac{k^{2}\log^{3}T}{T}\sqrt{\frac{c_{1}\log T}{T}}\mathbb{E}_{x_{t}\sim q_{t}}^{1/2}\left[\Vert\varepsilon_{t}\left(x_{t}\right)\Vert_{2}^{2}\right]
\end{align*}
as claimed. Here step (a) follows from Claim~\ref{claim:sde-K-1}
and Claim~\ref{claim:sde-K-2}; while step (b) utilizes Jensen's
inequality, and holds provided that $T$ is sufficiently large. 

\subsubsection{Proof of Claim \ref{claim:sde-K-1}\label{subsec:proof-lemma-sde-K-1}}

The term $K_{t,1}$ can be upper bounded by
\begin{align*}
\left|K_{t,1}\right| & =\left|\int_{\mathcal{A}_{t}}\left(\frac{p_{X_{t-1}|X_{t}}\left(x_{t-1}\mymid x_{t}\right)}{p_{Y_{t-1}^{\star}|Y_{t}}\left(x_{t-1}\mymid x_{t}\right)}-1\right)p_{Y_{t-1}^{\star}|Y_{t}}\left(x_{t-1}\mymid x_{t}\right)p_{X_{t}}\left(x_{t}\right)\left(x_{t-1}-\mu_{t}^{\star}\left(x_{t}\right)\right)^{\top}\varepsilon_{t}\left(x_{t}\right)\mathrm{d}x_{t-1}\mathrm{d}x_{t}\right|\\
 & \overset{\text{(i)}}{\leq}\int_{\mathcal{A}_{t}}\left|1-\frac{p_{X_{t-1}|X_{t}}\left(x_{t-1}\mymid x_{t}\right)}{p_{Y_{t-1}^{\star}|Y_{t}}\left(x_{t-1}\mymid x_{t}\right)}\right|p_{Y_{t-1}^{\star}|Y_{t}}\left(x_{t-1}\mymid x_{t}\right)p_{X_{t}}\left(x_{t}\right)\left|\left(x_{t-1}-\mu_{t}^{\star}\left(x_{t}\right)\right)^{\top}\varepsilon_{t}\left(x_{t}\right)\right|\mathrm{d}x_{t-1}\mathrm{d}x_{t}\\
 & \overset{\text{(ii)}}{\leq}C_{5}\frac{k^{2}\log^{3}T}{T}\int_{\mathcal{A}_{t}}p_{Y_{t-1}^{\star}|Y_{t}}\left(x_{t-1}\mymid x_{t}\right)p_{X_{t}}\left(x_{t}\right)\left|\left(x_{t-1}-\mu_{t}^{\star}\left(x_{t}\right)\right)^{\top}\varepsilon_{t}\left(x_{t}\right)\right|\mathrm{d}x_{t-1}\mathrm{d}x_{t}\\
 & \overset{\text{(iii)}}{=}C_{5}\frac{k^{2}\log^{3}T}{T}\mathbb{E}\left[\frac{\sigma_{t}^{\star}}{\sqrt{\alpha_{t}}}\left|Z^{\top}\varepsilon_{t}\left(X_{t}\right)\right|\ind\left\{ \left(X_{t},\frac{X_{t}+\eta_{t}s_{t}^{\star}\left(X_{t}\right)+\sigma_{t}^{\star}Z}{\sqrt{\alpha_{t}}}\right)\in\mathcal{A}_{t}\right\} \right]\\
 & \leq C_{5}\frac{k^{2}\log^{3}T}{T}\sqrt{\frac{\left(1-\alpha_{t}\right)\left(\alpha_{t}-\overline{\alpha}_{t}\right)}{\alpha_{t}\left(1-\overline{\alpha}_{t}\right)}}\mathbb{E}\left[\left|Z^{\top}\varepsilon_{t}\left(X_{t}\right)\right|\right]\overset{\text{(iv)}}{\leq}C_{5}\frac{k^{2}\log^{3}T}{T}\sqrt{\frac{8c_{1}\log T}{T}}\frac{2}{\sqrt{2\pi}}\mathbb{E}\left[\Vert\varepsilon_{t}(X_{t})\Vert_{2}\right]\\
 & \leq3C_{5}\frac{k^{2}\log^{3}T}{T}\sqrt{\frac{c_{1}\log T}{T}}\mathbb{E}_{x_{t}\sim q_{t}}\left[\Vert\varepsilon_{t}\left(x_{t}\right)\Vert_{2}\right].
\end{align*}
Here step (i) follows from Jensen's inequality; step (ii) utilizes
Lemma~\ref{lemma:Delta-SDE}; step (iii) follows from the definition
of $Y_{t}^{\star}$ in (\ref{eq:Y-star-defn}) and of $\mu_{t}^{\star}$
in (\ref{eq:defn-epsilon-mu}), and $Z_{t}\sim\mathcal{N}(0,I_{d})$
is independent of $X_{t}$; step (iv) follows from Lemma~\ref{lemma:step-size}
and the fact that $Z_{t}^{\top}\varepsilon_{t}(X_{t})\mymid X_{t}\sim\mathcal{N}(0,\Vert\varepsilon_{t}(X_{t})\Vert_{2}^{2})$
and hence
\[
\mathbb{E}\left[\left|Z_{t}^{\top}\varepsilon_{t}\left(X_{t}\right)\right|\right]=\mathbb{E}\left[\mathbb{E}\left[\left|Z_{t}^{\top}\varepsilon_{t}\left(X_{t}\right)\right|\mymid X_{t}\right]\right]=\frac{2}{\sqrt{2\pi}}\mathbb{E}\left[\Vert\varepsilon_{t}(X_{t})\Vert_{2}\right].
\]

\subsubsection{Proof of Claim \ref{claim:sde-K-2}\label{subsec:proof-lemma-sde-K-2}}

The term $K_{t,1}$ can be upper bounded by
\begin{align*}
\left|K_{t,2}\right| & \leq\int_{\mathcal{A}_{t}^{\mathrm{c}}}\left(p_{X_{t-1}|X_{t}}\left(x_{t-1}\mymid x_{t}\right)+p_{Y_{t-1}^{\star}|Y_{t}}\left(x_{t-1}\mymid x_{t}\right)\right)p_{X_{t}}\left(x_{t}\right)\left\Vert x_{t-1}-\mu_{t}^{\star}\left(x_{t}\right)\right\Vert _{2}\left\Vert \varepsilon_{t}\left(x_{t}\right)\right\Vert _{2}\mathrm{d}x_{t-1}\mathrm{d}x_{t}\\
 & \leq\underbrace{\left[\int_{\mathcal{A}_{t}^{\mathrm{c}}}\left(p_{X_{t-1}|X_{t}}\left(x_{t-1}\mymid x_{t}\right)+p_{Y_{t-1}^{\star}|Y_{t}}\left(x_{t-1}\mymid x_{t}\right)\right)p_{X_{t}}\left(x_{t}\right)\left\Vert x_{t-1}-\mu_{t}^{\star}\left(x_{t}\right)\right\Vert _{2}^{2}\mathrm{d}x_{t-1}\mathrm{d}x_{t}\right]^{1/2}}_{\eqqcolon\gamma_{1}}\\
 & \qquad\cdot\underbrace{\left[\int_{\mathcal{A}_{t}^{\mathrm{c}}}\left(p_{X_{t-1}|X_{t}}\left(x_{t-1}\mymid x_{t}\right)+p_{Y_{t-1}^{\star}|Y_{t}}\left(x_{t-1}\mymid x_{t}\right)\right)p_{X_{t}}\left(x_{t}\right)\left\Vert \varepsilon_{t}\left(x_{t}\right)\right\Vert _{2}^{2}\mathrm{d}x_{t-1}\mathrm{d}x_{t}\right]^{1/2}}_{\eqqcolon\gamma_{2}}.
\end{align*}
The second term $\gamma_{2}$ can be easily bounded by
\[
\gamma_{2}\leq\sqrt{2}\mathbb{E}_{x_{t}\sim q_{t}}^{1/2}\left[\Vert\varepsilon_{t}\left(x_{t}\right)\Vert_{2}^{2}\right].
\]
In what follows, we will bound the first term $\gamma_{1}$. Note
that
\begin{align*}
\gamma_{1}^{2} & =\underbrace{\int_{\mathcal{A}_{t}^{\mathrm{c}}}p_{X_{t-1},X_{t}}\left(x_{t-1},x_{t}\right)\left\Vert x_{t-1}-\mu_{t}^{\star}\left(x_{t}\right)\right\Vert _{2}^{2}\mathrm{d}x_{t-1}\mathrm{d}x_{t}}_{\eqqcolon\gamma_{1,1}}\\
 & \qquad+\underbrace{\int_{\mathcal{A}_{t}^{\mathrm{c}}}p_{Y_{t-1}^{\star}|Y_{t}}\left(x_{t-1}\mymid x_{t}\right)p_{X_{t}}\left(x_{t}\right)\left\Vert x_{t-1}-\mu_{t}^{\star}\left(x_{t}\right)\right\Vert _{2}^{2}\mathrm{d}x_{t-1}\mathrm{d}x_{t}}_{\eqqcolon\gamma_{1,2}}.
\end{align*}
We have
\begin{align*}
\gamma_{1,1} & =\mathbb{E}\left[\left\Vert X_{t-1}-\mu_{t}^{\star}\left(X_{t}\right)\right\Vert _{2}^{2}\ind\left\{ (X_{t},X_{t-1})\notin\mathcal{A}_{t}\right\} \right]\overset{\text{(i)}}{\leq}\mathbb{E}^{1/2}\left[\left\Vert X_{t-1}-\mu_{t}^{\star}\left(X_{t}\right)\right\Vert _{2}^{4}\right]\mathbb{P}^{1/2}\left((X_{t},X_{t-1})\notin\mathcal{A}_{t}\right)\\
 & \overset{\text{(ii)}}{\leq}\alpha_{t}^{-2}\mathbb{E}^{1/2}\left[\left\Vert \sqrt{1-\alpha_{t}}W_{t}+\eta_{t}^{\star}s_{t}^{\star}\left(X_{t}\right)\right\Vert _{2}^{4}\right]\exp\left(-\frac{C_{1}}{8}k\log T\right)\\
 & \overset{\text{(iii)}}{\leq}4\mathbb{E}^{1/2}\left[\left\Vert \sqrt{1-\alpha_{t}}W_{t}+\eta_{t}^{\star}s_{t}^{\star}\left(X_{t}\right)\right\Vert _{2}^{4}\right]\exp\left(-\frac{C_{1}}{8}k\log T\right)
\end{align*}
Here step (i) follows from Cauchy-Schwarz inequality; step (ii) follows
from Lemma~\ref{lemma:At-SDE} and the definition of $\mu_{t}^{\star}$
in (\ref{eq:defn-epsilon-mu}); while step (iii) uses the fact that
$\alpha_{t}\geq1/2$ (see Lemma~\ref{lemma:step-size}). Recall the
definition of $s_{t}^{\star}(\cdot)$
\begin{align*}
s_{t}^{\star}\left(x_{t}\right) & =-\frac{1}{1-\overline{\alpha}_{t}}\left(x_{t}-\sqrt{\overline{\alpha}_{t}}\mathbb{E}\left[X_{0}\mymid X_{t}=x_{t}\right]\right),
\end{align*}
which leads to the following upper bound
\begin{align}
\left\Vert s_{t}^{\star}\left(X_{t}\right)\right\Vert  & \leq\frac{1}{1-\overline{\alpha}_{t}}\left\Vert X_{t}\right\Vert _{2}+\frac{\sqrt{\overline{\alpha}_{t}}}{1-\overline{\alpha}_{t}}R=\frac{1}{1-\overline{\alpha}_{t}}\left\Vert \sqrt{\overline{\alpha}_{t}}X_{0}+\sqrt{1-\overline{\alpha}_{t}}\,\overline{W}_{t}\right\Vert _{2}+\frac{\sqrt{\overline{\alpha}_{t}}}{1-\overline{\alpha}_{t}}R\nonumber \\
 & \leq\frac{1}{\sqrt{1-\overline{\alpha}_{t}}}\left\Vert \overline{W}_{t}\right\Vert _{2}+2\frac{\sqrt{\overline{\alpha}_{t}}}{1-\overline{\alpha}_{t}}R.\label{eq:st-star-ub}
\end{align}
Hence we have
\begin{align*}
\mathbb{E}\left[\left\Vert \sqrt{1-\alpha_{t}}W_{t}+\eta_{t}^{\star}s_{t}^{\star}\left(X_{t}\right)\right\Vert _{2}^{4}\right] & \overset{\text{(i)}}{\leq}8\left(1-\alpha_{t}\right)^{2}\mathbb{E}\left[\left\Vert W_{t}\right\Vert _{2}^{4}\right]+\left(1-\alpha_{t}\right)^{4}\mathbb{E}\left[\left\Vert s_{t}^{\star}\left(X_{t}\right)\right\Vert _{2}^{4}\right]\\
 & \overset{\text{(ii)}}{\leq}8\left(\frac{c_{1}\log T}{T}\right)^{2}\mathbb{E}\left[\left\Vert W_{t}\right\Vert _{2}^{4}\right]+\left(\frac{8c_{1}\log T}{T}\right)^{4}\mathbb{E}\left[\left(\left\Vert \overline{W}_{t}\right\Vert _{2}+R\right)^{4}\right]\\
 & \overset{\text{(iii)}}{\leq}\frac{1}{16}\left(d^{2}+R^{4}\right).
\end{align*}
Here step (i) follows from the elementary inequality $8(x^{4}+y^{4})\geq(x+y)^{2}$;
step (ii) follows from Lemma~\ref{lemma:step-size} and (\ref{eq:st-star-ub});
step (iii) follows from $W_{t},\overline{W}_{t}\sim\mathcal{N}(0,I_{d})$
and the proviso that $T$ being sufficiently large. Hence we have
\[
\gamma_{1,1}\leq\sqrt{d^{2}+R^{4}}\exp\left(-\frac{C_{1}}{8}k\log T\right)\leq\exp\left(-\frac{C_{0}}{16}k\log T\right)
\]
as long as $C_{0}\gg c_{R}$ and $k\geq\log d$. Regarding $\gamma_{1,2}$,
we have
\begin{align*}
\gamma_{1,2} & \overset{\text{(i)}}{=}\mathbb{E}\left[\left\Vert \frac{X_{t}+\eta_{t}s_{t}^{\star}\left(X_{t}\right)+\sigma_{t}^{\star}Z_{t}}{\sqrt{\alpha_{t}}}-\frac{X_{t}+\eta_{t}^{\star}s_{t}^{\star}\left(X_{t}\right)}{\sqrt{\alpha_{t}}}\right\Vert _{2}^{2}\ind\left\{ (X_{t},X_{t-1})\notin\mathcal{A}_{t}\right\} \right]\\
 & =\frac{\sigma_{t}^{\star2}}{\alpha_{t}}\mathbb{E}\left[\left\Vert Z_{t}\right\Vert _{2}^{2}\ind\left\{ (X_{t},X_{t-1})\notin\mathcal{A}_{t}\right\} \right]\overset{\text{(ii)}}{\leq}\frac{\left(1-\alpha_{t}\right)\left(\alpha_{t}-\overline{\alpha}_{t}\right)}{\left(1-\overline{\alpha}_{t}\right)\alpha_{t}}\mathbb{E}^{1/2}\left[\left\Vert Z_{t}\right\Vert _{2}^{4}\right]\mathbb{P}^{1/2}\left((X_{t},X_{t-1})\notin\mathcal{A}_{t}\right)\\
 & \overset{\text{(iii)}}{\leq}\frac{8c_{1}\log T}{T}\exp\left(-\frac{C_{1}}{8}k\log T\right)\mathbb{E}^{1/2}\left[\left\Vert Z_{t}\right\Vert _{2}^{4}\right]\overset{\text{(iv)}}{\leq}\exp\left(-\frac{C_{1}}{16}k\log T\right).
\end{align*}
Here step (i) follows from the definition of $Y_{t}^{\star}$ in (\ref{eq:Y-star-defn})
and of $\mu_{t}^{\star}$ in (\ref{eq:defn-epsilon-mu}); step (ii)
follows from the Cauchy-Schwarz inequality; step (iii) utilizes Lemma~\ref{lemma:step-size}
and Lemma~\ref{lemma:At-SDE}; while step (iv) follows from $Z_{t}\sim\mathcal{N}(0,I_{d})$
and holds provided that $T$ is sufficiently large and $k\geq\log d$.
Taking the above bounds collectively yields
\[
\left|K_{t,2}\right|\leq\gamma_{1}\gamma_{2}\leq\sqrt{\gamma_{1,1}+\gamma_{1,2}}\gamma_{2}\leq2\exp\left(-\frac{C_{1}}{32}k\log T\right)\mathbb{E}_{x_{t}\sim q_{t}}^{1/2}\left[\Vert\varepsilon_{t}\left(x_{t}\right)\Vert_{2}^{2}\right].
\]

\section{Proof of Theorem~\ref{thm:uniqueness} \label{sec:proof-thm-unique}}

In view of the update rule (\ref{eq:forward-update}), the variables
$X_{0},X_{1},\ldots,X_{T}$ are jointly Gaussian, and we can check
from (\ref{eq:forward-formula}) that
\begin{equation}
X_{t}=\sqrt{\overline{\alpha}_{t}}X_{0}+\sqrt{1-\overline{\alpha}_{t}}\,\overline{W}_{t}\sim\mathcal{N}\left(0,\overline{\alpha}_{t}I_{k}+(1-\overline{\alpha}_{t})I_{d}\right),\label{eq:Xt-distribution}
\end{equation}
hence the score functions
\begin{equation}
s_{t}^{\star}(x)=-\left(\overline{\alpha}_{t}I_{k}+(1-\overline{\alpha}_{t})I_{d}\right)^{-1}x,\qquad\forall\,x\in\mathbb{R}^{d}.\label{eq:score-explicit}
\end{equation}
We first derive the density of $X_{t-1}$ conditional on $X_{t}=x_{t}$.
Since the joint distribution of $(X_{t-1},X_{t})$ is
\[
\left[\begin{array}{c}
X_{t-1}\\
X_{t}
\end{array}\right]\sim\mathcal{N}\left(\left[\begin{array}{c}
0\\
0
\end{array}\right],\left[\begin{array}{cc}
\overline{\alpha}_{t-1}I_{k}+(1-\overline{\alpha}_{t-1})I_{d} & \sqrt{\alpha_{t}}\left(\overline{\alpha}_{t-1}I_{k}+(1-\overline{\alpha}_{t-1})I_{d}\right)\\
\sqrt{\alpha_{t}}\left(\overline{\alpha}_{t-1}I_{k}+(1-\overline{\alpha}_{t-1})I_{d}\right) & \overline{\alpha}_{t}I_{k}+(1-\overline{\alpha}_{t})I_{d}
\end{array}\right]\right),
\]
we can derive that 
\[
X_{t-1}\mymid X_{t}=x_{t}\sim\mathcal{N}\left(\sqrt{\alpha_{t}}\Big(I_{k}+\frac{1-\overline{\alpha}_{t-1}}{1-\overline{\alpha}_{t}}\left(I_{d}-I_{k}\right)\Big)x_{t},\left(1-\alpha_{t}\right)\Big(I_{k}+\frac{1-\overline{\alpha}_{t-1}}{1-\overline{\alpha}_{t}}\left(I_{d}-I_{k}\right)\Big)\right).
\]
In addition, with perfect score estimation, we can use (\ref{eq:DDPM})
and (\ref{eq:score-explicit}) to achieve
\[
Y_{t-1}=\frac{Y_{t}+\eta_{t}s_{t}^{\star}\left(Y_{t}\right)+\sigma_{t}Z_{t}}{\sqrt{\alpha_{t}}}=\frac{Y_{t}-\eta_{t}\left(\overline{\alpha}_{t}I_{k}+(1-\overline{\alpha}_{t})I_{d}\right)^{-1}Y_{t}+\sigma_{t}Z_{t}}{\sqrt{\alpha_{t}}},
\]
which indicates that
\[
Y_{t-1}\mymid Y_{t}=x_{t}\sim\mathcal{N}\left(\frac{1}{\sqrt{\alpha_{t}}}\left(\left(1-\eta_{t}\right)I_{k}+\left(1-\frac{\eta_{t}}{1-\overline{\alpha}_{t}}\right)\left(I_{d}-I_{k}\right)\right),\frac{\sigma_{t}^{2}}{\alpha_{t}}I_{d}\right).
\]
Then we can check that for any $x_{t}\in\mathbb{R}^{d}$, 
\begin{align*}
 & \mathsf{KL}\left(p_{X_{t-1}|X_{t}}\left(\,\cdot\mymid x_{t}\right)\,\Vert\,p_{Y_{t-1}|Y_{t}}\left(\,\cdot\mymid x_{t}\right)\right)=\frac{\left(1-\alpha_{t}-\eta_{t}\right)^{2}}{2\sigma_{t}^{2}}\left\Vert I_{k}x_{t}\right\Vert _{2}^{2}+\frac{k}{2}\left(\frac{\alpha_{t}\left(1-\alpha_{t}\right)}{\sigma_{t}^{2}}-\log\frac{\alpha_{t}\left(1-\alpha_{t}\right)}{\sigma_{t}^{2}}-1\right)\\
 & \qquad+\frac{\left(1-\alpha_{t}-\eta_{t}\right)^{2}}{2\left(1-\overline{\alpha}_{t}\right)}\left\Vert \left(I_{d}-I_{k}\right)x_{t}\right\Vert _{2}^{2}+\frac{d-k}{2}\left(\frac{\left(1-\alpha_{t}\right)\left(\alpha_{t}-\overline{\alpha}_{t}\right)}{\sigma_{t}^{2}\left(1-\alpha_{t}\right)}-\log\frac{\left(1-\alpha_{t}\right)\left(\alpha_{t}-\overline{\alpha}_{t}\right)}{\sigma_{t}^{2}\left(1-\alpha_{t}\right)}-1\right).
\end{align*}
One can check that 
\[
z-\log z-1\geq0.1\min\left\{ 1,\left(z-1\right)^{2}\right\} ,\qquad\forall\,z>0.
\]
We combine the above two relations as well as the assumption that
$k\leq d/2$ to achieve
\[
\mathsf{KL}\left(p_{X_{t-1}|X_{t}}\left(\,\cdot\mymid x_{t}\right)\,\Vert\,p_{Y_{t-1}|Y_{t}}\left(\,\cdot\mymid x_{t}\right)\right)\geq\frac{\left(1-\alpha_{t}-\eta_{t}\right)^{2}}{2\left(1-\overline{\alpha}_{t}\right)}\left\Vert \left(I_{d}-I_{k}\right)x_{t}\right\Vert _{2}^{2}+\frac{d}{40}\left(\frac{\left(1-\alpha_{t}\right)\left(\alpha_{t}-\overline{\alpha}_{t}\right)}{\sigma_{t}^{2}\left(1-\alpha_{t}\right)}-1\right)^{2}.
\]
By taking expectation w.r.t.~$x_{t}$, we have
\[
\mathbb{E}_{x_{t}\sim q_{t}}\left[\mathsf{KL}\left(p_{X_{t-1}|X_{t}}\left(\,\cdot\mymid x_{t}\right)\,\Vert\,p_{Y_{t-1}|Y_{t}}\left(\,\cdot\mymid x_{t}\right)\right)\right]\geq\frac{d}{4}\left(1-\alpha_{t}-\eta_{t}\right)^{2}+\frac{d}{40}\left(\frac{\left(1-\alpha_{t}\right)\left(\alpha_{t}-\overline{\alpha}_{t}\right)}{\sigma_{t}^{2}\left(1-\alpha_{t}\right)}-1\right)^{2},
\]
where we use the fact that
\[
\mathbb{E}_{x_{t}\sim q_{t}}\big[\left\Vert \left(I_{d}-I_{k}\right)x_{t}\right\Vert _{2}^{2}\big]=\left(d-k\right)\left(1-\overline{\alpha}_{t}\right)\geq\frac{d}{2}\left(1-\overline{\alpha}_{t}\right).
\]

\begin{remark}
	In fact, for general target data distribution $p_{\mathsf{data}}$ satisfying the assumptions in Section~\ref{sec:setup}, we can start from an intermediate result \eqref{eq:proof-KL-main-1.5} from the proof of Theorem~\ref{thm:SDE} to show that
		\begin{align*}
				&\mathbb{E}_{x_{t}\sim q_{t}}\left[\mathsf{KL}\left(p_{X_{t-1}|X_{t}}\left(\,\cdot\mymid x_{t}\right)\,\Vert\,p_{Y_{t-1}|Y_{t}}\left(\,\cdot\mymid x_{t}\right)\right)\right]
				\geq \bigg(\frac{\sigma_{t}^{\star2}}{\sigma_t^2} + 2\log\frac{\sigma_t}{\sigma_t^\star}- 1\bigg)\frac{d}{2} + c_0\frac{(\eta_{t}-\eta_{t}^{\star})^2d}{2\sigma_{t}^{2}(1-\overline{\alpha}_t)} 
				\\
				&\qquad\qquad\qquad\qquad - c_{1}^2\frac{k^{4}\log^{6}T}{T^2}\bigg(3+\frac{\sigma_{t}^{\star2}}{\sigma_t^2}\bigg) - c_{1}\frac{k^{2}\log^{3}T}{T}\bigg|\frac{\sigma_{t}^{\star2}}{\sigma_t^2} - 1\bigg|\sqrt{d} - \exp(-c_2k\log T) 
		\end{align*}
		for some universal constant $c_0,c_1,c_2>0$. Notice the fact that $x^2 - 2\log x -1 \geq 0$ for any $x>0$, and the equality holds if and only if $x=1$. Therefore the above results suggests that, when both $d$ and $T$ are sufficiently large, unless $\eta_t=\eta_t^\star$ and $\sigma_t=\sigma_t^\star$, the corresponding denoising step will incur an error that is linear in $d$. Since the result in Theorem~\ref{thm:uniqueness} on Gaussian distribution already demonstrates this point, for simplicity we omit the proof of this result.
\end{remark}

\section{Technical lemmas}

This section collects a few useful technical tools that are useful
in the analysis.

\begin{lemma} \label{lemma:step-size}When $T$ is sufficiently large,
for $1\leq t\leq T$, we have

\[
\alpha_{t}\geq1-\frac{c_{1}\log T}{T}\geq\frac{1}{2}.
\]
In addition, for $2\leq t\leq T$, we have
\begin{align*}
\frac{1-\alpha_{t}}{1-\overline{\alpha}_{t}} & \leq\frac{1-\alpha_{t}}{\alpha_{t}-\overline{\alpha}_{t}}\leq\frac{8c_{1}\log T}{T}.
\end{align*}

\end{lemma}\begin{proof} See Appendix A.2 in \cite{li2023towards}.\end{proof}

\begin{lemma}\label{lemma:concentration}For $Z\sim\mathcal{N}(0,1)$
and any $t\geq1$, we know that
\[
\mathbb{P}\left(\left|Z\right|\geq t\right)\leq e^{-t^{2}/2},\qquad\forall\,t\geq1.
\]
In addition, for a chi-square random variable $Y\sim\chi^{2}(d)$,
we have
\[
\mathbb{P}(\sqrt{Y}\geq\sqrt{d}+t)\leq e^{-t^{2}/2},\qquad\forall\,t\geq1.
\]
\end{lemma}\begin{proof}See Proposition 2.1.2 in \cite{vershynin2018high}
and Section 4.1 in \cite{laurent2000adaptive}.\end{proof}

\begin{lemma}\label{lemma:initialization-error}Suppose that $T$
is sufficiently large. Then we have
\[
\mathsf{KL}\left(p_{X_{T}}\Vert p_{Y_{T}}\right)\leq T^{-100}.
\]
\end{lemma}\begin{proof}See Lemma 3 in \cite{li2023towards}.\end{proof}
\bibliographystyle{apalike}
\bibliography{reference-diffusion}

\end{document}